%% file: main.tex
\documentclass{article} %

\usepackage[left=3cm,right=3cm,top=4cm,bottom=4cm]{geometry}

\usepackage{authblk}

\input{math_commands.tex}

\usepackage{hyperref}
\usepackage{url}

\usepackage{tikz}
\usepackage{amsmath}

\renewcommand{\eqref}[1]{(\ref{#1})}

\usepackage[textsize=tiny]{todonotes}
\usepackage{amssymb}
\usepackage{mathtools}
\usepackage{amsthm}

\usepackage[utf8]{inputenc} %
\usepackage[T1]{fontenc}    %
\usepackage{hyperref}       %
\usepackage{url}            %
\usepackage{booktabs}       %
\usepackage{amsfonts}       %
\usepackage{nicefrac}       %
\usepackage{microtype}      %
\usepackage{xcolor}         %

\usepackage{booktabs}
\usepackage{soul}
\usepackage[numbers]{natbib}

\usepackage[justification=centering]{subfig}
\usepackage{graphicx}

\usepackage{float}          %

\usepackage{algorithm}
\usepackage{algorithmic}
\usepackage{bm}

\usepackage{cleveref}
\usepackage{fancyhdr}

\usepackage{enumitem}

\input{macros}
\theoremstyle{plain}
\newtheorem{theorem}{Theorem}[section]

\newtheorem{lemma}[theorem]{Lemma}

\theoremstyle{definition}

\theoremstyle{remark}

\title{Differentially Private Mixing of Public Datasets Improves Private Learning}

\date{}

\begin{document}

\author[1]{Yufei Chen}
\author[2]{Tejumade Afonja}
\author[1]{Anvith Thudi}
\author[1]{Nicolas Papernot}

\affil[1]{University of Toronto and Vector Institute, Toronto, Canada}
\affil[2]{CISPA Helmholtz Center for Information Security, Saarbrucken, Germany}

\maketitle

\begin{abstract}
\input{sections/abstract_v3}

\end{abstract}

\input{sections/introduction_v4}

\input{sections/related_work}

\input{sections/Preliminaries_v2}

\input{sections/method}

\input{sections/experiments}

\input{sections/conclusion}

\input{sections/acknowledgements}

\subsection*{AI use statement}

We only used generative AI tools for formatting the calculations presented in Appendix A.2, debugging code, formatting the figures and tables in this paper, and checking if there was literature on optimizing public pretraining data we missed. All code and references were verified by the authors. We take responsibility for the final content of this work, including text, claims or artifacts produced with the aid of generative AI.

\bibliographystyle{abbrvnat}
\bibliography{references}

\appendix

\input{sections/appendix}

\end{document}

%% file: math_commands.tex
\usepackage{amsmath,amsfonts,bm}

\def\eqref#1{equation~\ref{#1}}

\def\1{\bm{1}}

\DeclareMathAlphabet{\mathsfit}{\encodingdefault}{\sfdefault}{m}{sl}
\SetMathAlphabet{\mathsfit}{bold}{\encodingdefault}{\sfdefault}{bx}{n}

%% file: sections/abstract_v3.tex
Many machine learning applications involve sensitive data and therefore require training under differential privacy (DP). However, DP training often degrades model utility. In some cases, first pre-training the model on "public" data before finetuning with DP on the sensitive data can reduce the drop in utility. However, the success of this depends on how relevant the selected public dataset is to the sensitive data. We introduce the first pipeline that privately learns the mixture of several public datasets to pretrain on for a given sensitive downstream task. Our key insight is that we can privately find the best mixture of multiple public datasets by privately learning a low-dimensional linear model. We tested our method on the NIH dataset for X-ray classification and the ENRON email dataset for language modeling. Applying our method to find tailored mixtures of X-ray datasets to pretrain on for diseases in the NIH ChestX-ray14 dataset, we improved macro AUC by up to 0.037 across privacy budgets compared to the baselines, with gains as large as +22.8\% relative AUC on Cardiomegaly at $\epsilon=1$. For DP training on the ENRON dataset, pre-training on our mixture of The Common Pile (a collection of public-domain text datasets) decreased test perplexity by 16\% relative to the baseline mixtures.

%% file: sections/introduction_v4.tex
\section{Introduction}

\begin{figure}[t!]
  \centering
  \includegraphics[width=\linewidth]{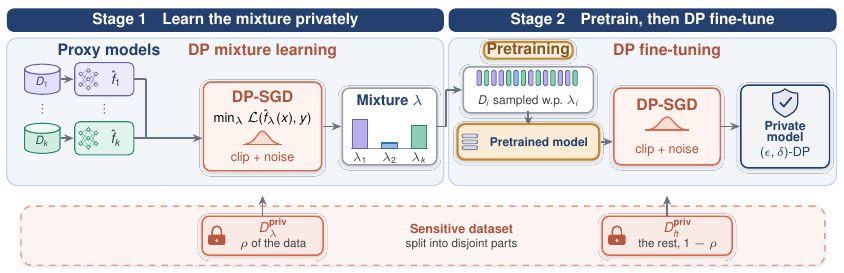}
  \caption{\textbf{Overview of DP-MixMin.} Proxy models trained on public data are combined with
  mixture weights $\lambda$ learned with DP-SGD on a fraction $\rho$ of the sensitive data
  (split A). We then pretrain on public data according to $\lambda$ and DP fine-tune on the rest
  (split B). }
  \label{fig:pipeline}
\end{figure}

For many machine learning tasks, the only data we have comes from sensitive sources, such as medical records~\citep{johnson2019mimiccxr} or conversations between individuals~\citep{klimt2004introducing}. Models trained on such data, even if the data is anonymized, are known to leak information about the individuals who contributed to the training dataset~\citep{shokri2017membership,carlini2021extracting}. Training with differential privacy (DP) is the de-facto framework for addressing this concern~\citep{dwork2006calibrating,dwork2014algorithmic}; DP guarantees that the models trained using an individual's data are indistinguishable to the models trained without their data. However, training with DP often comes with a 
large drop in model utility~\citep{abadi2016deep, de2022unlocking}. Past work has observed that first pre-training a model on public data (data that can be treated as non-sensitive, hence not contributing to privacy leakage) before DP finetuning the model on the sensitive data can sometimes reduce the drop in model utility~\citep{de2022unlocking,yu2022differentiallyprivatefinetuning}. Importantly, the effectiveness of this pre-training depends on how related the selected public data is to the target (sensitive) data~\citep{ben2010theory,isik2024scaling,pouget2024no, dominguez2024training, ganesh2023whyispublic, tramer2022considerations}. To the best of our knowledge, pre-training on all the public datasets uniformly is common practice and state-of-the art for privately mixing the public dataset: we improve this.

In this paper, we study how to privately choose the public pre-training mixture that gives the best downstream performance on a given sensitive dataset. Finding the best mixture privately can generically be approached with private hyperparameter selection techniques~\citep{liu2019privateselection,papernot2022hyperparameter}, where one trains on several different mixtures and then releases the best found mixture. However, this comes with significant computational costs: the space of mixtures which we need to trial (i.e., train on) grows exponentially with the number of public datasets. Compounding this, every candidate mixture we evaluate touches the sensitive data, so the more mixtures we try, the higher the privacy cost~\citep{papernot2022hyperparameter}. We show that it is possible to significantly improve downstream utility over standard approaches to public data mixing.

Our key insight is that we can privately learn the weighting of public datasets by privately optimizing a convex objective. This builds off MixMin~\citep{thudi2026mixminfindingdatamixtures}, a non-private data mixing method which showed that as models approach Bayes optimal (the best a model can do with enough compute and data), finding the best weighting of pre-training datasets for a downstream task reduces to minimizing a convex objective. We build onto this insight by realizing this objective is amenable to private learning with little utility degradation. Formally, if $k$ is the number of public datasets, by using existing DP algorithms~\citep{bassily2019private,feldman2020private,asi2021private,bassily2021differentially} we have the privacy cost for a fixed error grows at worst $O(\sqrt{k})$. In contrast, for certain data mixing problems, private hyperparameter search would need to try exponentially many mixtures, which translates to a privacy cost that grows linearly in $k$~\citep{papernot2022hyperparameter}. However, a key limitation of this convex objective is that it requires having models near Bayes optimal on each of the individual pre-training sources; we hence explore empirical approaches.

Our pipeline, DP-MixMin, empirically instantiates the DP convex optimization for data mixing alongside further DP finetuning (Figure~\ref{fig:pipeline}); our pipeline provides one end-to-end privacy guarantee.  We first split the sensitive data into two disjoint parts: one for mixture learning with a variant of DP-SGD~\citep{abadi2016deep}, and the other for DP fine-tuning the model pre-trained on the mixture. Concretely, we empirically learn the mixture by replacing the Bayes optimal models for each source with a proxy model trained on each individual public dataset (following past work~\citep{thudi2026mixminfindingdatamixtures}); these models do not see the private data, so leak no privacy. We then optimize the convex objective for data mixing with DP-SGD, providing a DP bound for the mixture. In the final stage we pre-train on this mixture, and then apply DP finetuning on the held out set. Because the two DP stages (mixing and fine-tuning) use disjoint datasets, parallel composition~\citep{dwork2014algorithmic} bounds the pipeline's overall privacy cost by the worse of the two steps, not their sum. So \emph{learning the mixture costs data, not privacy}. DP-MixMin fine-tunes on less data than our baseline of finetuning a model trained uniformly over all public datasets, %
but we show that better mixtures more than make up for it.

We empirically tested DP-MixMin across medical imaging and language tasks, and found across varying DP guarantees it consistently improved baselines (Figure~\ref{fig:pareto}). For medical imaging, we used the 14 disease classification tasks in the NIH chest X-ray dataset~\citep{Wang_2017} as sensitive datasets, and CheXpert~\citep{irvin2019chexpertlargechestradiograph} and MIMIC-CXR~\citep{johnson2019mimiccxr} as public sources. These sources also contain classification tasks over chest X-rays, but differ in the diseases and how the images were labelled. DP-MixMin consistently outperformed the baselines, with an average AUC gain of $\approx 0.03$, including a $22.8\%$ relative gain on the Cardiomegaly class. For language modeling, we used Enron~\citep{klimt2004introducing} as the sensitive dataset, targeting \emph{user-level} DP; protecting all of a user's emails, a strictly stronger and costlier guarantee than protecting single emails~\citep{kandpal2024user,levy2021learning,chua2024mind}. Pre-training on mixtures of The Common Pile~\citep{kandpal2025commonpilev018tb} decreased test perplexity by up to $16\%$, with similar gains under example-level DP on OpenWebMath~\citep{paster2023openwebmathopendatasethighquality}. We also found that DP-MixMin recovered mixtures close to the non-private solution in most settings, despite using a fraction of the sensitive data.

To conclude, our contributions are:

\begin{enumerate}
    \item DP-MixMin: The first end-to-end privatized pretraining mixture-learning pipeline. 
    \item Theoretically showing that when we train models near Bayes optimal, our procedure asymptotically outperforms the baseline of private hyperparameter search.
    \item Empirically showing that across medical imaging and language modeling tasks, DP-MixMin improves the performance of public-pretraining over past mixture finding approaches.
\end{enumerate}

\begin{figure}[t!]
  \centering
  \includegraphics[width=1.0\linewidth]{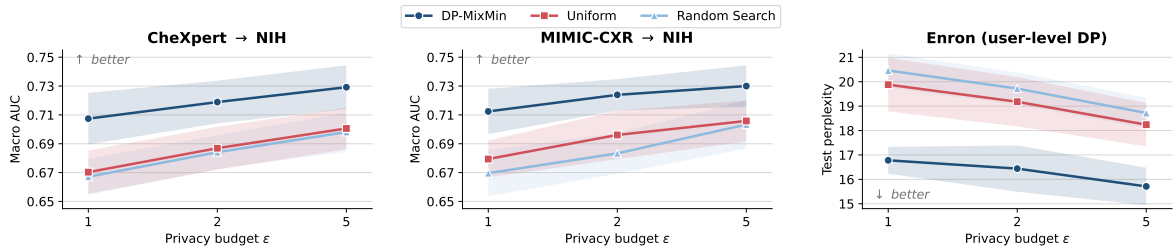}
  \caption{\textbf{DP-MixMin beats Uniform and Random Search at every $\varepsilon$.} Macro AUC on CheXpert/MIMIC-CXR$\rightarrow$NIH ($\uparrow$) and Enron test perplexity at $\rho=0.3$ ($\downarrow$).}
  \label{fig:pareto}
\end{figure}

%% file: sections/related_work.tex
\section{Related Work}
\label{sec:related}

We now discuss past work on DP finetuning and non-private mixing we build on, and the limitations of previous public data selection for DP learning: they could not improve over a uniform mixture.

\paragraph{DP public pretraining and non-private data mixing.} Public pretraining followed by private finetuning drives state-of-the-art DP results across language~\citep{li2022largelanguagemodels, yu2022differentiallyprivatefinetuning} and vision~\citep{de2022unlocking, mehta2022largescaletransfer}, and its effectiveness depends heavily on which public data is used~\citep{ganesh2023whyispublic, tramer2022considerations}. Non-privately, optimizing mixtures over many sources has shown to provide large gains in pre-training performance~\citep{gao2020pile800gbdatasetdiverse, xie2023doremi, liu2024regmix, gururangan2020dontstoppretraining, thudi2026mixminfindingdatamixtures}; among these, MixMin~\citep{thudi2026mixminfindingdatamixtures}, which we build on, showed that the mixture-learning objective is (asymptotically) equivalent to fitting a linear model over per-source proxy model predictions.

\paragraph{Prior work on private selection of public datasets.} \citet{gu2025choosingpublicdatasetsprivate} propose GSD, which scores a candidate public dataset's relevance by the distance between its gradients' top-k subspace and the private task's, and picks the single source with the smallest distance to the private task. %
However, GSD %
breaks down by construction when the public sources share the same input space $\mathcal{X}$: their scoring uses a single model's gradient on randomly labeled examples from each source, so two sources with identical inputs become indistinguishable regardless of their actual label distributions.\footnote{We also implemented GSD~\citep{gu2025choosingpublicdatasetsprivate}, and confirmed we could not obtain meaningful results in our setting; implementation details and results are provided in Appendix~\ref{app:gsd}.} \citet{gu2025choosingpublicdatasetsprivate} also present no end-to-end private experiments---the privatized version of the selection step is proposed but never empirically validated. Our method, DP-MixMin, can operate when input distributions are shared and provides DP guarantees. %
\citet{yu2024selectivepretrainingprivatefinetuning} train a DP domain classifier to rank examples within a single pre-chosen public corpus and pretrain on the top-ranked subset. While their method improves over the unfiltered baseline at smaller model scales, the benefit is no longer apparent at GPT2-Medium (355M parameters)~\footnote{Figure~2 in~\citet{yu2024selectivepretrainingprivatefinetuning} shows selective pretraining on 15\% of OpenWebText and full pretraining on 100\% of OpenWebText reaches the same perplexity.}. %
In contrast, DP-MixMin continues to deliver substantial gains over unweighted pretraining at GPT2-Medium scale. %

%% file: sections/Preliminaries_v2.tex
\section{Preliminaries}
\label{sec:prelim}

In this section, we review differential privacy (DP), data mixing, and relevant past results.

\paragraph{Differential privacy.}
A randomized mechanism $\mathcal{M} : \mathcal{D} \to \mathcal{R}$
satisfies $(\epsilon, \delta)$-differential privacy~\citep{dwork2006calibrating} if, for every pair of adjacent datasets $D, D' \in \mathcal{D}$
(differing in a single record) and every measurable set
$S \subseteq \mathcal{R}$,
\[
\Pr[\mathcal{M}(D) \in S] \;\leq\;
e^{\epsilon}\,\Pr[\mathcal{M}(D') \in S] \;+\; \delta,
\]
where $\epsilon \geq 0$ is the privacy budget and $\delta \in [0,1]$
bounds the probability of failure. 

We will use DP-SGD, which privatizes the SGD algorithm by clipping individual gradients and adding noise~\citep{abadi2016deep}, and the post-processing property of DP which states any function (independent of the private dataset) of a DP algorithm is also a DP mechanism with the same guarantees~\citep{dwork2014algorithmic}. We also leverage parallel composition, which states that if two DP mechanisms operate on disjoint parts of the private dataset, then the overall DP guarantee of releasing both mechanisms is at most the worst DP guarantee of the two mechanisms~\citep{mcsherry2009privacy}.

\paragraph{User-level differential privacy.}
\label{user-level-dp}
The definition above protects individual records (\emph{record-level} or
\emph{example-level} DP). However, a single individual may contribute multiple records, and record-level guarantees can degrade substantially for such individuals.
\emph{User-level} DP~\citep{mcmahan2018learning,levy2021learning}
aims to resolve this. Suppose each dataset $D$ is partitioned by users $D = \bigcup_{u \in U} D_u$, where $D_u$ is the (possibly multi-record)
contribution of user $u$. Two datasets $D, D'$ are \emph{user-adjacent}
if they only differ by the $D_u$ of a single user. A mechanism $\mathcal{M}$ satisfies
$(\epsilon,\delta)$ user-level DP if the inequality above holds for every
pair of user-adjacent datasets and every measurable $S \subseteq \mathcal{R}$. Note that while Record-level DP implies user-level DP through group privacy~\citep{dwork2014algorithmic}, this bound is
often loose. %
User-level guarantees are typically obtained directly without relying on record-level DP methods, e.g., by running DP-SGD~\citep{abadi2016deep} with sampling, clipping, and noise applied
per user rather than per example~\citep{mcmahan2018learning}. User-level DP satisfies both post-processing and parallel composition, provided the private dataset is split along user boundaries.

\paragraph{Data mixing setup.}
We consider learning from a collection of $k$ public datasets
$\{D_i^{pub}\}_{i=1}^{k}$ in order to improve performance on a private
dataset. Each public dataset $D_i^{pub}$ is assumed to be drawn
from an underlying source distribution $d_i$ over a
shared input--output domain $\mathcal{X} \times \mathcal{Y}$. The
private downstream dataset $D^{priv}$ is drawn from a target
distribution $d_{t}$, which is also over $\mathcal{X} \times \mathcal{Y}$.

Let $\mathcal{H}$ denote the hypothesis space for our models.
We represent models as functions $f : \mathcal{X} \to \mathcal{O}$, where
$\mathcal{O}$ is some output space (e.g., label probabilities), and we
evaluate performance using a loss function
$L : \mathcal{O} \times \mathcal{Y} \to \mathbb{R}_{+}$. We require the
loss function to be either cross-entropy (CE) or mean squared error
(MSE), as required for the MixMin algorithm which we
adapt~\citep{thudi2026mixminfindingdatamixtures}. Our goal is to learn sampling weights
$\lambda \in \Delta_k$ over the public datasets. %
Operationally, we sample an index $i \sim \lambda$ and then a data point from $D_i^{pub}$, giving samples from $d_\lambda = \sum_{i=1}^{k} \lambda_i\, d_i$.

To simplify data mixing formulations, we assume no covariate
shift among public sources: $d_i(x) = d_j(x)$ for all $i,j$. %
This assumption is used only to simplify how we frame the mixture problem, and it is also trivially satisfied in our experiments (i.e., generative tasks have no "covariates", and disease classification is done over the same distribution of images). %
In general, %
if there is covariate shift, one should use a more general data mixing objective that optimizes the joint distribution over the inputs and outputs to be the same as the target distribution (see Section 6.1 in~\citet{thudi2026mixminfindingdatamixtures}).

\paragraph{Data mixing as bi-level optimization.}
The data mixing problem asks which $\lambda \in \Delta_k$ yields a
model trained on $d_\lambda$ that performs best on $d_{t}$. This is
a bi-level optimization,
\begin{equation}
\min_{\lambda \in \Delta_k}\;
\mathbb{E}_{(x,y) \sim d_{t}}\!\bigl[L(f_\lambda(x),\, y)\bigr]
\quad\text{where}\quad
f_\lambda \in \arg\min_{f \in \mathcal{H}}\;
\mathbb{E}_{(x,y) \sim d_\lambda}\!\bigl[L(f(x),\, y)\bigr],
\label{eq:DM}
\end{equation}
that is generally hard to solve: the inner minimization over
$\mathcal{H}$ typically has to be run for many candidates $\lambda$. Beyond computational complexity, we note that private hyperparameter search (which optimizes this bi-level problem with random search) also has worse privacy-utility trade-offs. 

\paragraph{MixMin: a convex reduction.}
\citet{thudi2026mixminfindingdatamixtures} showed that under the assumptions above, as our model $f_{\lambda}$ approaches Bayes optimal for $d_{\lambda}$ (i.e., large compute and data limit),
\cref{eq:DM} reduces to a single convex problem in $\lambda$. Letting $f_i$ be the Bayes optimal model for each $d_i$, the optimal $\lambda$ solves:
\begin{equation}
\min_{\lambda \in \Delta_k}\;
\mathbb{E}_{(x,y) \sim d_{t}}
L\big(\sum_{i=1}^{k} \lambda_i\, f_i(x),\; y\big).
\label{eq:mixmin}
\end{equation}
In practice, it was found that running this optimization with empirically trained proxy models $\hat f_i$ still outperformed other mixture-finding methods~\citep{thudi2026mixminfindingdatamixtures}. %

%% file: sections/method.tex
\section{Differentially Private Data Mixing}

In this section we first highlight how privately optimizing the reduced data mixing objective (\cref{eq:mixmin}) leads to better utility-privacy trade-offs than private hyperparameter search. We then describe our empirical method for privately optimizing the reduced objective, which we call DP-MixMin. 

\subsection{Improved Utility over Private Hyperparameter Search Asymptotically}

Privately optimizing the MixMin reduction leads to asymptotically better privacy-utility
trade-offs than private hyperparameter search using Poisson sampling~\citep{papernot2022hyperparameter}
as we increase the number of public datasets $k$. We prove that there are cases where private
hyperparameter search requires the DP parameter $\epsilon$ to grow at least linearly with $k$, for a fixed error. In contrast, existing DP stochastic convex
optimization (SCO) algorithms applied to the MixMin reduction only require the DP guarantee
to grow with $\sqrt{k}$, meaning we can achieve better privacy at a given utility.
We state this formally in Lemma~\ref{lem:asymptotic_improvement}, where we suppress constants
that do not depend on $k$.

We quantify utility by the excess error of the recovered mixture $\hat{\boldsymbol\lambda}$
relative to the optimal mixture:
\begin{equation}
\label{eq:excess-error}
\alpha \;:=\; \mathbb{E}\big[F(\hat{\boldsymbol\lambda})\big] - F(\boldsymbol\lambda^\star),
\qquad
F(\boldsymbol\lambda) \;=\; \mathbb{E}_{(x,y)\sim d_t}\!\Big[L\Big(\textstyle\sum_{i=1}^{k}\lambda_i f_i(x),\,y\Big)\Big],
\end{equation}
where $F$ is the MixMin objective of \cref{eq:mixmin},
$\boldsymbol\lambda^\star=\arg\min_{\boldsymbol\lambda\in\Delta_k}F(\boldsymbol\lambda)$ is its
minimizer.

\begin{lemma}
\label{lem:asymptotic_improvement}
    Suppose the conditions for the MixMin convex data mixing reduction hold (\cref{eq:mixmin}). There exists a Lipschitz data mixing objective such that private hyperparameter search with
    Poisson sampling requires $\epsilon = \Omega(k)$ to achieve a fixed
    (sufficiently small) excess error $\alpha$ with constant
    probability. In contrast, there exist DP-SCO algorithms that, for any Lipschitz data mixing objective, have $\epsilon = O(\sqrt{k})$ to achieve excess error $\alpha$, for mixture-learning data size $n = \Omega(1/\alpha^2)$.
\end{lemma}

The key idea of the proof is to consider a specific Lipschitz data mixing problem where finding a good mixture is equivalent to finding a mixture within an $\ell_2$ ball around the best mixture. We then show the probability of uniformly sampling a mixture $\lambda \in \Delta^k$ that is also in this ball vanishes exponentially with $k$, which implies private hyperparameter search needs to try many random mixtures and hence use larger privacy budget $\epsilon$. In contrast, we are able to apply existing DP-SCO results to any reduced Lipschitz data mixing objectives to show we can achieve much less error at lower privacy budgets. %
The full proof is in Appendix~\ref{app:proofs}.
In Appendix~\ref{app:sqrtk} we empirically show our proposed DP data mixing algorithm (Algorithm~\ref{alg:dp-mixmin}) follows an  $O(\sqrt{k})$ excess error rate.

While Lemma~\ref{lem:asymptotic_improvement} shows optimizing \cref{eq:mixmin} should lead to better utility, this assumes we can train models close to Bayes optimal. We turn to how we empirically approximate \cref{eq:mixmin}.

\subsection{Implementing DP-MixMin}
\label{ssec:implementation}

We now explain how we empirically optimize the MixMin objective (\cref{eq:mixmin}), train the proxy models $\hat{f}_i$ which we use instead of the Bayes optimal $f_i$, and how we allocate the privacy budget for private mixture finding and finetuning. The overall pipeline is summarized in Figure~\ref{fig:pipeline}.

\paragraph{Making MixMin Differentially Private.} %

To empirically optimize \cref{eq:mixmin} with DP, %
we replace the Bayes optimal models $f_i$ with models $\hat{f}_i$ trained empirically on $D_i^{pub}$, and the expectation over $d_t$ with the empirical expectation over a dataset $D^{priv}_{\lambda}\sim d_t$. %
To now optimize this we run a post-processing of the DP-SGD~\citep{abadi2016deep} update: for each iteration, per-example gradients w.r.t $\lambda$ are computed, clipped to a fixed norm, and perturbed with Gaussian noise before being aggregated. The resulting noisy gradient is then passed to the entropic-descent update. The overall procedure satisfies $(\epsilon,\delta)$-DP by the DP guarantees of the sampled Gaussian mechanism~\citep{mironov2019r}, post-processing~\citep{dwork2014algorithmic}, and composition theorems to accumulate DP guarantees across steps~\citep{mironov2017renyi}~\footnote{In particular, we use the DP accountant for DP-SGD provided by Opacus~\citep{opacus}}. See Algorithm~\ref{alg:dp-mixmin}, with blue text describing changes to the non-private optimization described by \citet{thudi2026mixminfindingdatamixtures}. See Algorithm~\ref{alg:user-wise-dp-mixmin} in the Appendix for user-level DP.

\begin{algorithm}[t]
\caption{\textcolor{blue}{DP}-MixMin}
\label{alg:dp-mixmin}
\begin{algorithmic}[1]
\REQUIRE \textcolor{blue}{Target privacy guarantees ($\epsilon$,$\delta$)}, step size $\eta$, number of steps $n$, loss function $L$ (CE or MSE), \textcolor{blue}{private mixture-learning split $D_{\lambda}^{priv} \subseteq \mathcal{X} \times \mathcal{Y}$}, proxy models $\{\hat{f}_i : i = 1, \ldots, k\}$, \textcolor{blue}{clipping norm $C$}
\STATE Initialize $\lambda_i \leftarrow \tfrac{1}{k}$ for all $i = 1, \ldots, k$
\STATE Pre-compute $\{\hat{f}_i(x) : x \in D_{\lambda}^{priv}, \, i = 1, \ldots, k\}$ 
\FOR{$j = 1$ to $n$}
    \STATE $\hat{f}_\lambda(x) \leftarrow \sum_{i=1}^{k} \lambda_i \hat{f}_i(x)$
    \STATE $\ell(x, y) \leftarrow L(\hat{f}_\lambda(x), y)$ for all $(x, y) \in D_{\lambda}^{priv}$
    \STATE $g(x, y) \leftarrow \nabla_\lambda \ell(x, y)$
    \STATE \textcolor{blue}{$\tilde{g}(x, y) \leftarrow \text{Clip}(g(x, y), C)$} \COMMENT{\textcolor{blue}{clip per-example gradients}}
    \STATE $g_{\text{sum}} \leftarrow \sum_{(x, y) \in D_{\lambda}^{priv}} \textcolor{blue}{\tilde{g}(x, y)}$
    \STATE \textcolor{blue}{$z \sim \mathcal{N}(0, \sigma^2 I)$} \COMMENT{\textcolor{blue}{noise calibrated to $(\epsilon, \delta)$}}
    \STATE \textcolor{blue}{$\bar{g} \leftarrow \tfrac{1}{|D_{\lambda}^{priv}|} (g_{\text{sum}} + z)$}
    \STATE $\lambda_i \leftarrow \lambda_i \, e^{-\eta \textcolor{blue}{\bar{g}_i}}$ for all $i = 1, \ldots, k$
    \STATE $\lambda \leftarrow \lambda \,/\, \sum_{i=1}^{k} \lambda_i$ \COMMENT{project onto $\Delta_k$}
\ENDFOR
\STATE Return $\lambda$
\end{algorithmic}
\end{algorithm}

\paragraph{Reducing Proxy Model Costs.} %
Past work on empirically optimizing \cref{eq:mixmin} required training a separate proxy model on each public dataset~\citep{thudi2026mixminfindingdatamixtures}. We reduced costs for classification tasks by training a single multi-class proxy model to jointly predict the labels for each public dataset: $\hat{f}: \mathcal{X} \rightarrow \mathcal{O}^k$. To do so, we concatenated the label spaces of all $k$ public datasets (representing labels as one-hot or multi-hot vectors depending on the datasets), and trained %
the proxy model to output a probability vector over this combined label space. The proxy model for the $i$-th dataset is then obtained as $\hat{f}_i(x) = \hat{f}(x)[i]$, i.e., the $i$-th coordinate of the output. If learned to Bayes optimality, this proxy would provide the same output probabilities as a model trained only on $D_i^{pub}$. We found this significantly reduced training and storage costs when scaling the number of public datasets to mix. For generative language modeling, we followed the standard procedure of training an individual model on each public dataset. %
Further details are provided in Section~\ref{ssec:lang_models}.

\paragraph{Privacy budget allocation.}

We split the private target dataset $D^{priv}$ into two disjoint subsets: $D_{\lambda}^{priv}$, which is used to compute the mixture weights (Algorithm~\ref{alg:dp-mixmin}), and $D_{ft}^{priv}$, which is used to fine-tune the pre-trained model. By parallel composition~\citep{dwork2014algorithmic}, if the mechanism used
for mixture learning is $(\epsilon_1, \delta_1)$-DP and the
mechanism used for finetuning is $(\epsilon_2, \delta_2)$-DP, %
then the overall guarantee is: 
$(\epsilon, \delta) \;=\;
\bigl(\max(\epsilon_1, \epsilon_2),\, \max(\delta_1, \delta_2)\bigr)$. Note that the baseline of pre-training on the uniform mixture uses all of $D^{priv}$ for finetuning; we found that DP-MixMin improves over uniform mixtures, showing that we can use less data for finetuning if we optimize the mixture.

%% file: sections/experiments.tex
\section{Experiments}

We compared the full DP-MixMin pipeline---private mixture learning, mixture-weighted public pretraining, and DP finetuning---%
against three baselines:
\begin{enumerate}
    \item \textbf{Random Initialization}: privately fine-tune with no public pretraining, isolating the impact of public pretraining.
    \item \textbf{Uniform Pretraining}: pretrain with equal weight on every public source, isolating the value of learning a target-specific public-data mixture.
    \item \textbf{Random Search} (non-private): select the best mixture from a random sample of candidate mixtures~\citep{bergstra2012random, papernot2022hyperparameter} by training a proxy model on each one and evaluating performance on the downstream task. The selection of the best mixture from the random sample was done non-privately, so this over-estimates the performance of private hyperparameter search. Full details are in Appendix~\ref{app:rand_search}.%
\end{enumerate}
We evaluated these methods' downstream utility on the private dataset, across privacy budgets, sources of public data, and modalities. We also compare DP-MixMin to non-private MixMin.

\subsection{Image Classification}
\label{sec:vision_experiments}
We evaluated DP-MixMin for chest X-ray disease classification. %
Prior work applied DP-SGD to train models from scratch on these datasets~\citep{tayebiarasteh2024securing, gu2025choosingpublicdatasetsprivate}, which corresponds to our random initialization baseline. %
We found mixtures of other chest X-ray datasets (for different diseases) so that pretraining on them improves performance.
\subsubsection{Experimental Setup}

\paragraph{Datasets and Models.} 
We considered two different collections of pretraining datasets: (1) CheXpert~\citep{irvin2019chexpertlargechestradiograph} and (2) MIMIC-CXR~\citep{johnson2019mimiccxrjpglargepubliclyavailable}, both standard public sources for chest X-ray transfer learning~\citep{rajpurkar2017chexnetradiologistlevelpneumoniadetection, gu2025choosingpublicdatasetsprivate}. Note each of these datasets also consist of 14 chest X-ray disease classification tasks, but differ in the diseases and how the images were labelled. %
We used a ResNet-50 as the proxy model and a randomly initialized ResNet-18 for downstream pretraining and finetuning. The smaller downstream model reduced the cost of pretraining, allowing us to run more experiments within our compute budget.

\paragraph{Data preprocessing and computational usage} For all three chest X-ray datasets, CheXpert, MIMIC-CXR, and NIH, uncertain labels ($-1$) were mapped to $0$ (negative); $0$ and $1$ were kept as-is. The total compute used was $\sim$ 20,000 GPU hours over a cluster of L40 GPUs.

\paragraph{Methodology.} For each of the 14 downstream tasks, 80\% of the NIH data was used for training and 20\% was split between validation and test. DP-MixMin and Random Search accessed the private task only through parallel composition, %
using 10\% of the training split for mixture learning and the rest for finetuning. Uniform pretraining used all of the training split for DP finetuning. We repeated each method $5$ times with a different random partition of the NIH data into mixture-finding, finetuning, validation, and test sets, as well as a different random seed for model initialization and DP noise. See Appendix~\ref{app:img_exp_details} for details on training hyperparameters. We write $\rho = |D^{\mathrm{priv}}_\lambda| / |D^{\mathrm{priv}}|$ for the fraction used for
private mixture-learning: %
see Appendix~\ref{app:choosing-rho} for guidance on choosing it.

\subsubsection{Image Classification Results}

\begin{figure}[t]
\centering
\includegraphics[width=0.99\linewidth]{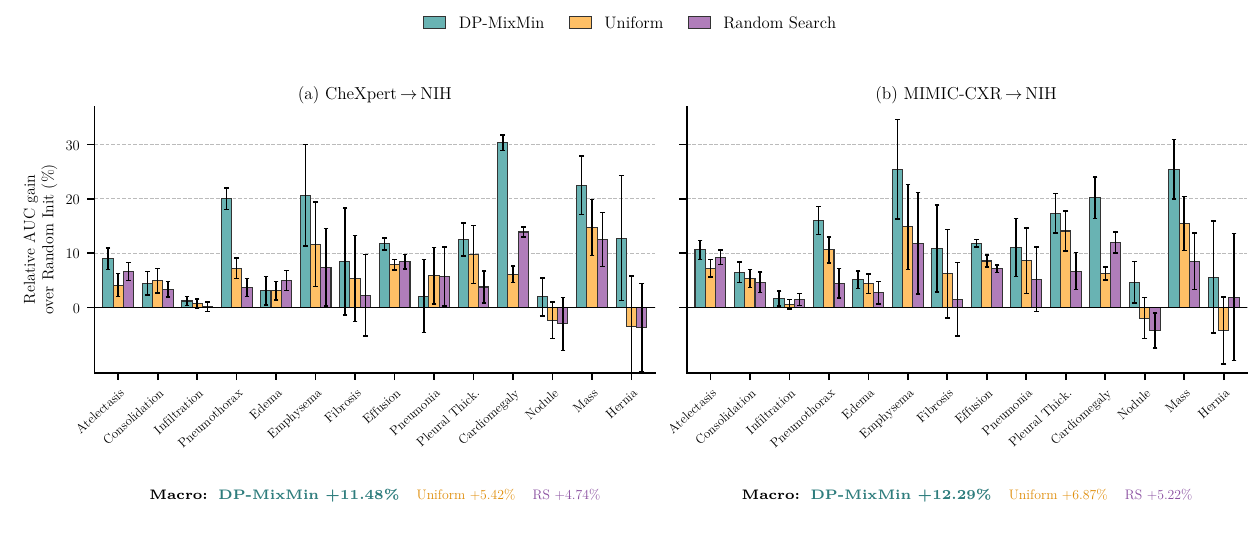}   %
\caption{\textbf{Per-task relative AUC gain over Random Init on NIH Chest X-ray at $\epsilon=1$.}
Bars show $\big(\mathrm{AUC(method)}-\mathrm{AUC(Random\ Init)}\big)/\mathrm{AUC(Random\ Init)}$ as a percentage, with propagated standard deviation over 3--5 seeds. The macro AUC gain (unweighted mean across 14 tasks) is reported next to each subplot title; DP-MixMin roughly doubles, and in most cases more than doubles, the gain of both Uniform and Random Search on each public pretraining source. Full raw AUCs are reported in Tables~\ref{tab:chexpert_to_nih_full} and~\ref{tab:mimic_to_nih_full}.}
\label{fig:nih_delta_over_random_init}
\end{figure}

\paragraph{Improvements were consistent across downstream tasks.} Across the 14 NIH tasks, we found DP MixMin always improved or was on-par (within one standard deviation) with the baselines of uniform pretraining and random search (Figure~\ref{fig:nih_delta_over_random_init}). This was the case whether CheXpert or MIMIC-CXR was used for the public datasets. %
In the case of Cardiomegaly, DP-MixMin's gain over Random Init was roughly $5\times$ Uniform's at $\epsilon=1$.

\paragraph{Gains were consistent across privacy budgets} DP-MixMin also outperformed the baselines at every privacy budget we tested: $+0.037, +0.032, +0.029$ absolute macro AUC gain over uniform at $\epsilon = 1, 2, 5$ for CheXpert pretraining(Figure~\ref{fig:pareto}), and $+0.033, +0.028, +0.024$ for MIMIC-CXR.  Tables~\ref{tab:chexpert_to_nih_full} and~\ref{tab:mimic_to_nih_full} in the Appendix report absolute AUC across settings.

\subsection{Language Models}
\label{ssec:lang_models}

We evaluated DP-MixMin on two downstream tasks: Enron~\citep{enron} email completion with user-level DP and OpenWebMath~\citep{paster2023openwebmathopendatasethighquality} with example-level DP. For Enron we compared to uniform pretraining and random search, and for OpenWebMath we only compared to uniform as it had often outperformed the other baseline (see Figure~\ref{fig:nih_delta_over_random_init} and Table~\ref{tab:enron_results}).

\subsubsection{Experimental Setup}
\paragraph{Datasets and Models.} We used 16 domains from The Common Pile~\citep{kandpal2025commonpilev018tb} as candidate public pretraining sources, selected from the 18 domains containing at least 10GB of data~\footnote{Two were excluded due to streaming limitations.}. For each domain we trained a GPT-2 Small (124M parameters) proxy on documents with at least 1024 GPT-2 tokens until roughly 45M tokens had been seen. %
The downstream model was GPT-2 Medium (355M parameters). %
DP finetuning was only done on the last two layers and output head of the model, due to memory allocation constraints on our GPUs.

\paragraph{Data preprocessing and computational usage.} For Enron, we restricted to the \texttt{\_sent\_mails} and \texttt{sent\_mails} folders, yielding 148 users and approximately 80{,}000 emails. This %
avoided multi-author records such as received emails, forwarded messages, and duplicated threads, which require more complex multi-attribution privacy treatment~\citep{ponomareva2025dpfydatapracticalguide}. For OpenWebMath, we streamed 50{,}000 samples to form the dataset. The total compute was $\approx5{,}000$ L40 GPU-hours.

\paragraph{Methodology} For Enron, we ablated $\rho = |D^{\mathrm{priv}}_\lambda| / |D^{\mathrm{priv}}|$ between $\{10\%, 20\%, 30\%\}$. For OpenWebMath, DP-MixMin used $10\%$. See Appendix~\ref{app:lang_exp_details} for further details.

\subsubsection{Language Results}

\begin{table}[t]
\centering
\small
\caption{\textbf{Enron test perplexity (mean $\pm$ std) after DP finetuning at $\delta{=}10^{-5}$.} 
DP-MixMin beats both Uniform and Random Search at every privacy budget. 
At $\epsilon{=}1$, the mixture quality (and stability) improves monotonically with the mixture weight finding 
fraction $\rho$.}
\label{tab:enron_results}
\begin{tabular}{ll ccc}
\toprule
\textbf{Method} & \textbf{$\rho$} & \textbf{$\epsilon{=}1$} & \textbf{$\epsilon{=}2$} & \textbf{$\epsilon{=}5$} \\
\midrule
\textsc{DP-MixMin}        & $0.3$ & $\mathbf{16.78 \pm 0.55}$ & $\mathbf{16.44 \pm 0.95}$ & $\mathbf{15.71 \pm 0.77}$ \\
\textsc{DP-MixMin}        & $0.2$ & $17.52 \pm 1.55$ & --- & --- \\
\textsc{DP-MixMin}        & $0.1$ & $17.97 \pm 2.72$ & --- & --- \\
\midrule
\textsc{Uniform}       & ---   & $19.88 \pm 1.10$ & $19.18 \pm 1.01$ & $18.24 \pm 0.89$ \\
\textsc{Random Search} & $0.3$ & $20.46 \pm 0.70$ & $19.72 \pm 0.66$ & $18.72 \pm 0.62$ \\
\bottomrule
\end{tabular}
\end{table}

\begin{table}[t]
\centering
\small
\caption{\textbf{OpenWebMath test perplexity (mean $\pm$ std) after DP finetuning at $\delta{=}10^{-5}$.}
DP-MixMin lowers test perplexity by roughly $13\%$ over uniform Common Pile pretraining at $\epsilon{=}1$.
The uniform baseline is essentially flat in $\epsilon$ on this domain, showing %
DP-MixMin at $\epsilon =1$ is still better.}
\label{tab:openwebmath_results}
\begin{tabular}{ll ccc}
\toprule
\textbf{Method} & \textbf{$\rho$} & \textbf{$\epsilon{=}1$} & \textbf{$\epsilon{=}2$} & \textbf{$\epsilon{=}5$} \\
\midrule
\textsc{DP-MixMin}  & $0.1$ & $\mathbf{12.22 \pm 0.17}$ & --- & --- \\
\midrule
\textsc{Uniform} & ---   & $14.05 \pm 0.17$ & $14.04 \pm 0.17$ & $14.02 \pm 0.17$ \\
\bottomrule
\end{tabular}
\end{table}

\paragraph{DP-MixMin's utility advantage grew under tighter $\epsilon$}: As shown in Table~\ref{tab:enron_results}, the relative perplexity reduction of DP-MixMin over the baselines widened as $\epsilon$ shrank on Enron: $13.9\%$ at $\epsilon{=}5$, $14.3\%$ at $\epsilon{=}2$, and $15.6\%$ at $\epsilon{=}1$. On OpenWebMath, we found the performance of uniform mixture weights was stable across $\epsilon = 1,2,5$, and DP-MixMin improved over it at $\epsilon  = 1$ (Table~\ref{tab:openwebmath_results}).

\paragraph{Public pretraining mixture mattered more than amount of finetuning data.} %
While each of the mixtures at ratios $\rho \in {0.1, 0.2, 0.3}$ outperformed uniform pretraining, allocating private data toward mixture estimation mattered more than reserving it for finetuning for Enron. The proxy-fraction ablation in Table~\ref{tab:enron_results} exhibited monotonic improvement with $\rho$ in both mean perplexity ($17.97 \to 17.52 \to 16.78$) and cross-seed standard deviation ($2.72 \to 1.55 \to 0.45$). %

\subsection{Comparing DP-MixMin to Non-Private MixMin}

\paragraph{Preventing privacy leakage.}
We audited whether the data mixture $\lambda$ leaked private information in
practice, following the canary-based auditing framework of \citet{jagielski2020auditing} %
using both natural and worst-case (max-clip)
canaries. Across all 14 NIH tasks and both pretraining sources, the DP parameter measured by the audit,
$\epsilon_{LB}$, stayed far below the claimed $\epsilon=1$, with the highest being $0.131$. %
Removing the DP noise from the
same audit recovered $\epsilon_{LB}$ as high as $5.6$, confirming that this step is genuinely privacy-sensitive and that DP is necessary (Tables~\ref{tab:audit-chexpert},
\ref{tab:audit-mimic}; further details in Appendix~\ref{app:audit}).

\paragraph{DP-MixMin algorithm often recovered non-private MixMin weights.} To understand the impact of DP on mixture finding, we compared the DP-MixMin weights $\boldsymbol{\lambda}_{\mathrm{DP}}$ to the non-private MixMin $\boldsymbol{\lambda}^{\star}$ (which used the entire $D^{priv}$ training split for mixture finding). %
Across modalities, DP-MixMin recovered mixtures close to $\boldsymbol{\lambda}^{\star}$ even with small privacy budgets and limited private data (Tables~\ref{tab:vision-weights-all},~\ref{tab:enron-weights},~\ref{tab:owm-weights} in the Appendix). By $\epsilon=1$, DP-MixMin achieved cosine similarity of  $0.97$ and $0.93$ for CheXpert$\to$NIH and MIMIC-CXR$\to$NIH respectively. %
On Enron, DP-MixMin reached cosine similarity %
$0.92$ having seen only 12 users at $\epsilon=2$ user-level DP. %
On OpenWebMath, using only $2.8$M tokens ($\rho=0.1$), DP-MixMin achieved cosine similarity above $0.9$ at $\epsilon=0.2$. %

\paragraph{DP-MixMin's downstream performance was on-par with non-private MixMin.}
Given the similarity in the weights, we further checked whether it translated to matched downstream utility. We compared the full DP-MixMin pipeline at $\varepsilon=1$ to a non-private MixMin baseline that paid no privacy for mixture finding, and also reused the \emph{entire} private dataset (both the mixing and finetuning splits) for weight-finding and DP finetuning steps. %
Across both pretraining sources, DP-MixMin was comparable to non-private MixMin: the macro AUC difference between DP-MixMin and non-private MixMin was only $0.717$ vs.\ $0.723$ for MIMIC-CXR and $0.711$ vs.\ $0.712$ for CheXpert. Most per-class differences lay within one standard deviation (Tables~\ref{tab:mimic-nonpriv-vs-dp},~\ref{tab:chexpert-nonpriv-vs-dp}).

%% file: sections/conclusion.tex
\section*{Conclusion}

In this paper, we initiated the study of privately optimizing public-data pretraining mixtures for a target sensitive dataset. Across medical imaging and language modeling, DP-MixMin improved private downstream utility over uniform pretraining and random mixture search, with large gains under low-prevalence labels, tight privacy budgets, and user-level privacy. Extending DP-MixMin
to tabular and other domain settings remains important future work. In particular, public tabular sources often differ in their input features, violating the no-covariate-shift assumption our reduction relies on (Section~\ref{sec:prelim}). Future work could also look into how to improve our current DP-MixMin pipeline. Initial results showed our approach was robust to less proxy models training (Appendix~\ref{app:proxy-robustness}) and whether parallel or sequential composition was used (Appendix~\ref{app:sequential}). However, the theoretical results of \citet{bassily2021differentially} suggest using DP Franke-Wolfe instead of DP mirror descent.

%% file: sections/acknowledgements.tex
\section*{Acknowledgements} We acknowledge the following sponsors, who support our
research with financial and in-kind contributions: CIFAR
through the Canada CIFAR AI Chair, NSERC through the
Discovery Grant and an Alliance Grant with CSE and wirh ServiceNow
and DRDC, the Schmidt Sciences foundation through the
AI2050 Early Career Fellow program. Anvith Thudi is supported by a Vanier Fellowship from the Natural Sciences
and Engineering Research Council of Canada. Resources used
in preparing this research were provided, in part, by the
Province of Ontario, the Government of Canada through
CIFAR, and companies sponsoring the Vector Institute. We
thank the computing team at the University of Toronto’s
Computer Science Department for administering and procuring the compute infrastructure used for the experiments in
this paper.

%% file: sections/appendix.tex
\section{Appendix}

\subsection{Proof of Lemma~\ref{lem:asymptotic_improvement}}
\label{app:proofs}

\begin{proof}
    We first provide a specific data mixing problem (under the lemma assumptions) that requires private hyperparameter search to have $\epsilon = \Omega(k)$ for fixed excess error with constant probability. We then show, under the lemma assumptions, the reduced MixMin objective always has $\epsilon = O(\sqrt{k})$.

    To show worst-case scaling for private hyperparameter search, we need only find one "hard" example. Consider the case where the data mixture objective is $\|\lambda - \lambda^*\|_2^2$ for $\lambda^*$ in the relative interior of the simplex (i.e., $\lambda^*_i \neq 0$ for any $i$). This corresponds to taking $L$ to be MSE, $f_i(x) = e_i$ for all $x$, and $y = \lambda^*$ always (which hence defines the target distribution). Then for random search to achieve $\alpha$ error, we need $\|\lambda - \lambda^*\|_2 \leq \sqrt{\alpha}$. Note the probability density function for uniform sampling on the simplex is $(k-1)!$ (where $k$ is the number of public datasets we are mixing). Thus the probability of a random mixture being in $S_{\alpha} = \{\lambda : \|\lambda - \lambda^*\|_2 \leq \sqrt{\alpha} , \lambda \in \Delta^{k}\}$ is $(k-1)! \, Vol(S_{\alpha}) = \frac{(k-1)! \, \pi^{(k-1)/2}}{\Gamma((k-1)/2+1)} \alpha^{(k-1)/2}$ for $\alpha$ sufficiently small. This follows from noting that in a sufficiently small ball around any point in the interior of the simplex, the simplex is a hyperplane defined by $\langle \mathbf{1},\lambda\rangle = 1$, and the intersection of a hyperplane around a $k$-dim sphere centered at a point on it is yet another sphere of dimension $k-1$ and the same radius (by rotational symmetry).

    Now recall Poisson hyperparameter search for mixture finding first samples a number $r$ following the Poisson distribution, and then trains on $r$ random mixtures and picks the best. As $\alpha \rightarrow 0$, we have the probability that the closest of the $r$ random mixtures is in $S_{\alpha}$, where $r$ is also a random variable, is $\Theta\!\left(\mathbb{E}r \, \frac{(k-1)! \, \pi^{(k-1)/2}}{\Gamma((k-1)/2+1)} \, \alpha^{(k-1)/2}\right)$ (see Appendix A.2.1 in \citet{papernot2022hyperparameter} for the first-order approximation which becomes tight as $\alpha \rightarrow 0$). Hence for the probability to not vanish with $k$, we need $\mathbb{E}r = \Omega\!\left(\left(\tfrac{1}{\alpha}\right)^{(k-1)/2}\right)$.

    Plugging this into Theorem 6 in \citet{papernot2022hyperparameter} shows the $\epsilon$ guarantee for $(\epsilon,\delta)$ grows $\propto \log(\mathbb{E} r) = \Omega(k)$.

    In contrast, the MixMin reduction presents a Lipschitz  DP-SCO objective. It is known that an achievable expected excess error vs. privacy trade-off is $\alpha = O(\frac{1}{\sqrt{n}} + \frac{\sqrt{k\ln(1/\delta)}}{\epsilon n})$~\citep{bassily2019private,feldman2020private,asi2021private}~\footnote{In fact, in the specific case of generalized linear losses, this can be far better in dimension dependence~\citep{bassily2021differentially}}. We hence need $\epsilon = O(\sqrt{k}/(n\alpha)) = O(\sqrt{k}\cdot\alpha) = O(\sqrt{k})$, noting $n=\Omega(1/\alpha^2)$ by assumption and $\alpha\leq 1$. Note that by Lipschitzness and the fact the simplex has $\ell_2$ diameter bounded by 1, the loss of the mixture is always bounded by a constant independent of $k$, and so always has at least constant probability around the expectation (independent of $k$), and we can conclude $O(\sqrt{k})$ is sufficient for excess error less than $\alpha$ with constant probability.
\end{proof}

\subsection{Guidance on choosing the mixture-selection budget \texorpdfstring{$\rho$}{rho}}
\label{app:choosing-rho}

DP-MixMin sets aside a fraction $\rho$ of the private data to estimate the mixture
weights. In the main text we validate a choice of $\rho$ \emph{a posteriori}, by
checking that the privately recovered mixture $\boldsymbol{\lambda}_{\mathrm{DP}}$ has
high cosine similarity to the non-private solution $\boldsymbol{\lambda}^\star$
(Tables~\ref{tab:vision-weights-all}--\ref{tab:weights-pneumonia}). Here we give a
complementary \emph{a-priori} guideline, derived directly from our error bound
(Lemma~\ref{lem:asymptotic_improvement}), that tells a practitioner how large $\rho$
must be before running anything.

\paragraph{Setup.} We measure quality by the excess error $\alpha$ of
\cref{eq:excess-error}, here instantiated with the empirically trained proxies
$\hat f_i$ in place of the Bayes-optimal $f_i$. Write $T_\lambda = \rho N$ for the number
of private examples used for mixing, with $N$ the private set size and $k$ the number of
public sources.

\paragraph{An a-priori bound on \texorpdfstring{$\rho$}{rho}.}
The DP-SCO trade-off for this objective (Lemma~\ref{lem:asymptotic_improvement}) is
$\alpha = O\!\big(1/\sqrt{T_\lambda} + \sqrt{k\ln(1/\delta)}/(\epsilon\,T_\lambda)\big)$.
Requiring each term to contribute at most $\alpha/2$ yields a sufficient budget:
\begin{align}
    &\alpha = O\!\left(\frac{1}{\sqrt{T_\lambda}} + \frac{\sqrt{k\ln(1/\delta)}}{\epsilon\, T_\lambda}\right), \\
    &\text{Requiring each term} \leq \alpha/2: \quad
      \frac{1}{\sqrt{T_\lambda}}\leq\frac{\alpha}{2} \;\Rightarrow\; T_\lambda\geq\frac{4}{\alpha^2}, \\
    &\qquad \frac{\sqrt{k\ln(1/\delta)}}{\epsilon T_\lambda}\leq\frac{\alpha}{2}
      \;\Rightarrow\; T_\lambda\geq\frac{2\sqrt{k\ln(1/\delta)}}{\epsilon\alpha}, \\
    &\therefore\quad
      T_\lambda \;\geq\; \max\!\left(\frac{4}{\alpha^2},\ \frac{2\sqrt{k\ln(1/\delta)}}{\epsilon\alpha}\right),
      \qquad \rho = \frac{T_\lambda}{N}.
    \label{eq:rho-rule}
\end{align}
This is a guideline rather than a single universal constant. The required budget is
governed by two regimes: a \emph{statistics-limited} regime where $T_\lambda \gtrsim 4/\alpha^2$
(the cost of estimating the objective from finite data), and a \emph{privacy-limited}
regime where $T_\lambda \gtrsim 2\sqrt{k\ln(1/\delta)}/(\epsilon\alpha)$ (the cost of the DP
noise). The two cross over at $\epsilon^\star = \tfrac{\alpha}{2}\sqrt{k\ln(1/\delta)}$: for
$\epsilon > \epsilon^\star$ the statistical term dominates and $\rho$ is independent of the
privacy budget; for $\epsilon < \epsilon^\star$ the required $\rho$ grows as $1/\epsilon$.

\paragraph{Reading the guideline at the budgets we use.}
For $\epsilon \geq 1$ we operate in the statistics-limited regime (e.g. for our vision
setting $k=14$, $\delta=10^{-5}$ give $\epsilon^\star \approx 0.63 < 1$), so the achievable
error is of order $\alpha \approx 1/\sqrt{T_\lambda}$. Evaluated at $\rho = 0.1$, this single
relation reproduces both of our empirical operating points:
\begin{itemize}
    \item \textbf{Vision (per-example DP).} $\rho=0.1$ gives $T_\lambda \approx 8{,}970$
    examples, so $\alpha \approx 1/\sqrt{T_\lambda} \approx 0.01$, consistent with the strong
    recovery we observe (Table~\ref{tab:vision-weights-all}: cosine similarity $0.97$ at
    $\epsilon=1$ across the 14 NIH tasks).
    \item \textbf{Language (example-level DP).} $\rho=0.1$ gives $\approx 2.8$M OpenWebMath
    tokens, so $\alpha \approx 1/\sqrt{T_\lambda} \approx 0.001$, consistent with the
    near-exact recovery we observe (Table~\ref{tab:owm-weights}: cosine similarity $1.00$
    at $\epsilon=1$).
\end{itemize}
The achievable $\alpha$ therefore tracks data availability: with millions of tokens,
language reaches a far tighter $\alpha$ than vision, whose mixture-learning splits contain
only thousands of examples. In both cases $\rho=0.1$ clears the threshold of
\cref{eq:rho-rule}, which is why it is a robust default; a practitioner facing a new
dataset can invert the same relation to size $\rho$ for a target $\alpha$.

\subsection{Extended Experimental Details}

\subsubsection{Image Classification}
\label{app:img_exp_details}

\textbf{DP-MixMin Implementation.} Proxy models were trained for 4 and 10 epochs respectively for CheXpert MIMIC-CXR at learning rate $10^{-4}$ and batch size 256 (the largest batch size that fit on a single L40 GPU). DP-MixMin used clipping norm $13.0$, selected from $\{1.0, 2.0, 5.0, 10.0, 13.0, 15.0, 20.0\}$ with a batch size of 1000 selected from $\{500, 1000, 1500\}$, and learning rate chosen from $\{0.5, 1.0, 2.0\}$.

\textbf{Training details.} Pretraining and private fine-tuning on the ResNet18s both used batch size 200 (the largest batch that fit on an L40). Given mixture weights $\lambda \in \Delta^{14}$, we pretrained the ResNet-18 by minimizing the binary cross-entropy loss between its scalar output and the mixture-weighted target $\sum_{c=1}^{14} \lambda_c \, y_c$, where $y_c$ was the per-disease binary label.~\footnote{Note that this is the importance weighting defined by the mixture, and it reduces variance relative to sampling $y_c$ according to $\lambda_c$.} %
We evaluated the pretrained ResNet-18 at 3, 6, and 9 epochs and selected the checkpoint with the highest downstream validation AUC before fine-tuning. Fine-tuning ran for 36{,}000 steps, with learning rate $10^{-4}$ and DP-SGD clipping norm $2.0$. All hyperparameters were tuned once on a single pilot configuration---Atelectasis %
classification as the NIH downstream task with CheXpert pretraining at $\varepsilon = 1$---and then held fixed across all 14 tasks, both pretraining sources, all seeds, and all three privacy budgets $\varepsilon \in \{1, 2, 5\}$ at $\delta = 10^{-5}$. 

\subsubsection{Language Modeling}
\label{app:lang_exp_details}

\textbf{DP-MixMin Implementation.} For Enron, DP-MixMin was run for $200$ steps with per-sample clipping norm $0.01$ and step size $5.0$. For OpenWebMath, DP-MixMin was run for $200$ steps with clipping norm $0.2$ and step size $5.0$. 

\textbf{Training details.} Pretraining uses GPT-2 Medium from random initialization for $7.1\!\times\!10^9$ tokens at context length $1024$, effective batch size $64$, and learning rate $5\cdot 10^{-3}$. Enron fine-tuning uses $5$ users per gradient step with $20$ accumulation steps (effective batch of $100$ users), G\_ULS of $15$, context length $512$, and learning rate $5\cdot 10^{-4}$ (chosen as the largest stable rate from $\{5\cdot 10^{-5}, 5\cdot 10^{-4}, 10^{-3}, 5\cdot 10^{-3}\}$). Per-device batch sizes were set to the largest power of $2$ that fit in available GPU memory, with gradient accumulation used to reach the target effective batch. Private Enron fine-tuning uses clipping norm $0.5$ and a token budget of $25$M. Private OpenWebMath fine-tuning used batch size $500$ at learning rate $5\cdot 10^{-3}$ for pretraining and a $50$M-token fine-tuning budget. All results are averaged over $3$ seeds. 

\subsection{On the infeasibility of GSD in our setting}
\label{app:gsd}

\citet{gu2025choosingpublicdatasetsprivate} propose Gradient Subspace Distance (GSD), which scores a candidate public source by training a model on randomly labeled examples from both the public source and the private task, then measuring the distance between their gradient subspaces. The lowest-distance source is selected.

GSD is inapplicable to our setting in Section~\ref{sec:vision_experiments}. We binarize each of CheXpert and MIMIC-CXR along their 14 disease axes (Section~\ref{sec:vision_experiments}), so all 14 candidate sources within a corpus share the same image set $\mathcal{X}$ and differ only in per-image labels. With random labels, gradient subspaces depend only on $\mathcal{X}$, so any two sources sharing inputs receive a GSD of zero, and the candidates DP-MixMin must rank are precisely such sources. GSD is, by construction, only informative when candidates have distinct input distributions, for example disjoint vision benchmarks of the kind on which~\citet{gu2025choosingpublicdatasetsprivate} evaluate it.

The coarsest GSD variant with signal in our setting selects between \emph{whole corpora}: CheXpert versus MIMIC-CXR. Either choice is dominated by DP-MixMin, which improves over uniform pretraining on each corpus independently (Table~\ref{tab:group-gains}). Even granting GSD an oracle picker over the two corpora, the resulting pipeline still falls short of DP-MixMin.

\subsection{Robustness to Proxy Model Quality}
\label{app:proxy-robustness}

Lemma~\ref{lem:asymptotic_improvement} assumes proxy models near Bayes optimal. To check how much this matters in practice, we varied the number of epochs used to train the CheXpert proxy (1 to 6, versus 4 in our main experiments) and ran DP-MixMin for NIH Atelectasis at $\epsilon = 1$. Results were similar across all proxy budgets (Table~\ref{tab:proxy-epochs}), and all remained well above uniform pretraining ($0.707$ AUC), suggesting small changes in proxy quality have little effect on DP-MixMin's downstream performance.

\begin{table}[h]
\centering
\small
\caption{\textbf{DP-MixMin is insensitive to proxy training budget.} NIH Atelectasis
AUC (mean $\pm$ std over 3 seeds) with CheXpert pretraining at $\epsilon = 1$.}
\label{tab:proxy-epochs}
\begin{tabular}{lc}
\toprule
\textbf{Proxy epochs} & \textbf{Atelectasis AUC} \\
\midrule
1 & $0.744 \pm 0.005$ \\
2 & $0.745 \pm 0.005$ \\
4 (main experiments) & $0.742 \pm 0.006$ \\
6 & $0.749 \pm 0.003$ \\
\bottomrule
\end{tabular}
\end{table}

\subsection{Empirically Validating Lemma~\ref{lem:asymptotic_improvement}}
\label{app:sqrtk}

To show that Lemma~\ref{lem:asymptotic_improvement} holds empirically under its assumptions, we analyze the excess error as a function of the number of public sources being mixed, $k$, and trace its trajectory as $k$ increases.

\paragraph{Setup.} We train ResNet-50 proxies on $k$ ImageNet-1k classes~\citep{imagenet15russakovsky} ($k\in\{2, 4, 8, 16, 64, 128, 256, 512,1000\}$, the $k$ classes being the $k$ sources) and run DP-MixMin on the first 10 binarized CIFAR-100 classes~\citep{Krizhevsky09learningmultiple} at $\varepsilon{=}1$, $\delta{=}10^{-5}$ over all $N{=}50{,}000$ images.

\begin{figure}[h]
  \centering
  \includegraphics[width=0.5\textwidth]{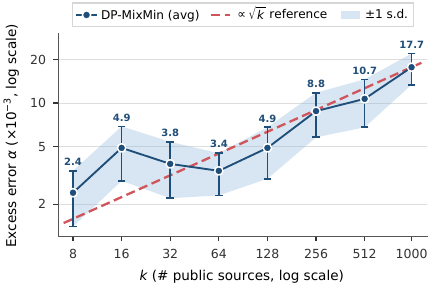}
  \caption{\textbf{Excess error scales roughly as $\sqrt{k}$ for large $k$.}
  Average excess error of DP-MixMin over the 10 CIFAR-100 tasks as the number
  of public sources $k$ increases (both axes log scale). Error bars and the
  shaded band show $\pm1$ standard deviation across tasks. The dashed red line
  is a $\sqrt{k}$ reference; the fitted slope for $k \geq 32$ is $0.49$,
  consistent with the predicted $\sqrt{k}$ rate.}
  \label{fig:sqrtk}
\end{figure}

\paragraph{Metric.} We report the excess error $\alpha = L(\boldsymbol{\lambda}_{\mathrm{DP}}) - L(\boldsymbol{\lambda}^\star)$ (\cref{eq:excess-error}) against the non-private optimum $\boldsymbol{\lambda}^\star$, averaged over 5 seeds and the first 10 binarized CIFAR100 tasks 

\paragraph{Result.} Excess error grows sub-linearly in $k$ (Table~\ref{tab:sqrtk-perclass}). A log-log fit over $k\geq32$ gives slope $0.49$, matching the $\sqrt{k}$ prediction ($0.5$) and well below the linear cost of private hyperparameter search.
\

\begin{table}[h]
\centering \scriptsize \setlength{\tabcolsep}{3pt}
\caption{\textbf{DP-MixMin excess error $\alpha=L(\boldsymbol{\lambda}_{\mathrm{DP}})-L(\boldsymbol{\lambda}^\star)$ across the number of public sources $k$} (ImageNet-1k $\to$ CIFAR-100, $\varepsilon{=}1$, $\delta{=}10^{-5}$), for all 10 CIFAR-100 one-vs-rest tasks. Mean $\pm$ std over 5 seeds, in units of $10^{-3}$. The last column (\textbf{Avg}) averages over all 5 seeds and 10 tasks.}
\label{tab:sqrtk-perclass}
\resizebox{\textwidth}{!}{%
\begin{tabular}{r *{10}{c} @{\hskip 10pt} c}
\toprule
$k$ & class 0 & class 1 & class 2 & class 3 & class 4 & class 5 & class 6 & class 7 & class 8 & class 9 & \textbf{Avg} \\
\midrule
2 & $0.1{\scriptstyle\pm0.1}$ & $0.1{\scriptstyle\pm0.1}$ & $0.1{\scriptstyle\pm0.2}$ & $0.1{\scriptstyle\pm0.2}$ & $0.1{\scriptstyle\pm0.1}$ & $0.1{\scriptstyle\pm0.2}$ & $0.1{\scriptstyle\pm0.1}$ & $0.1{\scriptstyle\pm0.2}$ & $0.1{\scriptstyle\pm0.1}$ & $0.1{\scriptstyle\pm0.2}$ & $\mathbf{0.1{\scriptstyle\pm0.1}}$ \\
4 & $0.1{\scriptstyle\pm0.0}$ & $0.0{\scriptstyle\pm0.0}$ & $0.0{\scriptstyle\pm0.0}$ & $0.1{\scriptstyle\pm0.0}$ & $0.0{\scriptstyle\pm0.0}$ & $0.0{\scriptstyle\pm0.0}$ & $0.0{\scriptstyle\pm0.0}$ & $0.0{\scriptstyle\pm0.0}$ & $0.0{\scriptstyle\pm0.0}$ & $0.1{\scriptstyle\pm0.0}$ & $\mathbf{0.0{\scriptstyle\pm0.0}}$ \\
8 & $1.1{\scriptstyle\pm0.1}$ & $1.4{\scriptstyle\pm0.2}$ & $1.3{\scriptstyle\pm0.2}$ & $3.7{\scriptstyle\pm0.2}$ & $3.3{\scriptstyle\pm0.2}$ & $2.1{\scriptstyle\pm0.1}$ & $2.3{\scriptstyle\pm0.2}$ & $2.8{\scriptstyle\pm0.1}$ & $4.1{\scriptstyle\pm0.2}$ & $2.1{\scriptstyle\pm0.1}$ & $\mathbf{2.4{\scriptstyle\pm1.0}}$ \\
16 & $5.7{\scriptstyle\pm0.4}$ & $7.0{\scriptstyle\pm0.2}$ & $3.3{\scriptstyle\pm0.4}$ & $2.9{\scriptstyle\pm0.3}$ & $2.4{\scriptstyle\pm0.4}$ & $7.0{\scriptstyle\pm0.6}$ & $4.2{\scriptstyle\pm0.3}$ & $2.8{\scriptstyle\pm0.4}$ & $5.8{\scriptstyle\pm0.5}$ & $7.8{\scriptstyle\pm0.5}$ & $\mathbf{4.9{\scriptstyle\pm2.0}}$ \\
32 & $2.8{\scriptstyle\pm0.2}$ & $6.7{\scriptstyle\pm0.5}$ & $3.4{\scriptstyle\pm0.3}$ & $2.0{\scriptstyle\pm0.1}$ & $1.9{\scriptstyle\pm0.1}$ & $3.4{\scriptstyle\pm0.2}$ & $3.7{\scriptstyle\pm0.2}$ & $2.7{\scriptstyle\pm0.3}$ & $5.8{\scriptstyle\pm0.3}$ & $5.3{\scriptstyle\pm0.3}$ & $\mathbf{3.8{\scriptstyle\pm1.6}}$ \\
64 & $3.8{\scriptstyle\pm0.2}$ & $2.1{\scriptstyle\pm0.3}$ & $3.7{\scriptstyle\pm0.2}$ & $3.4{\scriptstyle\pm0.3}$ & $2.6{\scriptstyle\pm0.3}$ & $4.4{\scriptstyle\pm0.3}$ & $4.4{\scriptstyle\pm0.2}$ & $5.0{\scriptstyle\pm0.5}$ & $1.8{\scriptstyle\pm0.1}$ & $2.3{\scriptstyle\pm0.3}$ & $\mathbf{3.4{\scriptstyle\pm1.1}}$ \\
128 & $5.9{\scriptstyle\pm0.5}$ & $1.6{\scriptstyle\pm0.3}$ & $5.7{\scriptstyle\pm0.3}$ & $5.2{\scriptstyle\pm0.4}$ & $6.4{\scriptstyle\pm0.4}$ & $2.7{\scriptstyle\pm0.2}$ & $7.4{\scriptstyle\pm0.2}$ & $6.8{\scriptstyle\pm1.0}$ & $3.8{\scriptstyle\pm0.4}$ & $3.6{\scriptstyle\pm0.7}$ & $\mathbf{4.9{\scriptstyle\pm1.9}}$ \\
256 & $12.2{\scriptstyle\pm0.7}$ & $2.7{\scriptstyle\pm0.3}$ & $9.2{\scriptstyle\pm0.3}$ & $10.6{\scriptstyle\pm0.3}$ & $11.5{\scriptstyle\pm0.3}$ & $4.6{\scriptstyle\pm0.4}$ & $10.9{\scriptstyle\pm0.2}$ & $10.5{\scriptstyle\pm0.7}$ & $7.9{\scriptstyle\pm0.6}$ & $8.3{\scriptstyle\pm0.9}$ & $\mathbf{8.8{\scriptstyle\pm3.0}}$ \\
512 & $8.6{\scriptstyle\pm0.4}$ & $6.7{\scriptstyle\pm0.7}$ & $4.5{\scriptstyle\pm0.5}$ & $16.2{\scriptstyle\pm0.3}$ & $10.1{\scriptstyle\pm0.6}$ & $10.3{\scriptstyle\pm0.4}$ & $13.7{\scriptstyle\pm0.4}$ & $15.7{\scriptstyle\pm0.9}$ & $13.9{\scriptstyle\pm0.5}$ & $7.0{\scriptstyle\pm1.0}$ & $\mathbf{10.7{\scriptstyle\pm3.9}}$ \\
1000 & $18.6{\scriptstyle\pm1.0}$ & $14.9{\scriptstyle\pm0.9}$ & $16.3{\scriptstyle\pm0.7}$ & $23.4{\scriptstyle\pm0.3}$ & $14.9{\scriptstyle\pm0.7}$ & $9.6{\scriptstyle\pm0.4}$ & $21.9{\scriptstyle\pm0.3}$ & $22.3{\scriptstyle\pm0.5}$ & $20.6{\scriptstyle\pm0.5}$ & $14.8{\scriptstyle\pm1.3}$ & $\mathbf{17.7{\scriptstyle\pm4.3}}$ \\
\bottomrule
\end{tabular}}
\end{table}

\subsection{Sequential vs.\ Parallel Composition}
\label{app:sequential}

Instead of splitting $D^{priv}$ into disjoint $D^{priv}_{\lambda}$ and $D^{priv}_{ft}$
(parallel composition), one can use all of $D^{priv}$ for both steps and split the
privacy budget, giving $(\epsilon_1 + \epsilon_2, \delta_1 + \delta_2)$-DP by sequential
composition~\citep{dwork2014algorithmic}. We compared the two on NIH with CheXpert
pretraining at total $\epsilon = 1$: sequential used $\epsilon_1 = 0.1$ and
$\epsilon_2 = 0.9$, and parallel used our main setting ($\rho = 0.1$,
$\epsilon_1 = \epsilon_2 = 1$). Neither was consistently better, and both were close to
non-private MixMin (Table~\ref{tab:sequential}). This is expected: DP noise in each step
scales roughly with $1/(\epsilon \cdot n)$ (Appendix~\ref{app:choosing-rho}), and with
$\epsilon_1 = \rho\,\epsilon$ and $\epsilon_2 = (1-\rho)\,\epsilon$, sequential
composition matches parallel composition's $\epsilon \cdot n$ in both steps, differing
only in using somewhat more data per step.

\begin{table}[h]
\centering \small
\caption{\textbf{Learned mixture is stable across proxy training budget.} Cosine similarity of the DP-MixMin mixture (early-stop, $\varepsilon{=}1$) at proxy epochs $\{1,2,6\}$ to the main $E{=}4$ mixture, per NIH class (mean $\pm$ std, 3 seeds). $E{=}4$ is the self-reference ($1.000$).}
\label{tab:proxy-epoch-cosine}
\begin{tabular}{l ccc}
\toprule
\textbf{Class} & \textbf{E1} & \textbf{E2} & \textbf{E6} \\
\midrule
Atelectasis        & $0.868 \pm 0.015$ & $0.986 \pm 0.003$ & $0.972 \pm 0.014$ \\
Consolidation      & $0.886 \pm 0.004$ & $0.941 \pm 0.035$ & $0.944 \pm 0.012$ \\
Infiltration       & $0.895 \pm 0.018$ & $0.873 \pm 0.046$ & $0.848 \pm 0.023$ \\
Pneumothorax       & $0.952 \pm 0.019$ & $0.976 \pm 0.003$ & $0.975 \pm 0.003$ \\
Edema              & $0.978 \pm 0.003$ & $0.965 \pm 0.003$ & $0.887 \pm 0.029$ \\
Emphysema          & $0.986 \pm 0.006$ & $0.988 \pm 0.002$ & $0.970 \pm 0.009$ \\
Fibrosis           & $0.993 \pm 0.005$ & $0.989 \pm 0.002$ & $0.916 \pm 0.024$ \\
Effusion           & $0.870 \pm 0.020$ & $0.884 \pm 0.006$ & $0.938 \pm 0.006$ \\
Pneumonia          & $0.975 \pm 0.003$ & $0.984 \pm 0.004$ & $0.889 \pm 0.008$ \\
Pleural Thickening & $0.982 \pm 0.011$ & $0.988 \pm 0.001$ & $0.931 \pm 0.013$ \\
Cardiomegaly       & $0.904 \pm 0.004$ & $0.975 \pm 0.007$ & $0.726 \pm 0.067$ \\
Nodule             & $0.967 \pm 0.010$ & $0.972 \pm 0.012$ & $0.932 \pm 0.031$ \\
Mass               & $0.934 \pm 0.021$ & $0.982 \pm 0.004$ & $0.974 \pm 0.009$ \\
Hernia             & $0.983 \pm 0.001$ & $0.998 \pm 0.000$ & $0.914 \pm 0.007$ \\
\bottomrule
\end{tabular}
\end{table}

\begin{table}[h]
\centering
\small
\caption{\textbf{Sequential and parallel composition perform similarly.} NIH AUC
(mean $\pm$ std over 3 seeds), CheXpert pretraining, total $\epsilon = 1$.}
\label{tab:sequential}
\begin{tabular}{lccc}
\toprule
\textbf{Class} & \textbf{Sequential} & \textbf{Parallel} & \textbf{Non-private} \\
\midrule
Atelectasis & $0.740 \pm 0.010$ & $0.747 \pm 0.005$ & $0.746 \pm 0.006$ \\
\bottomrule
\end{tabular}
\end{table}

\subsection{Auditing DP-MixMin's Privacy Guarantee}
\label{app:audit}

\paragraph{Threat model.}
We consider example-level $(\epsilon,\delta)$-differential privacy, where neighboring
private datasets differ by the addition or removal of one example. We partition the
private dataset as $D^{\text{priv}} = D_{\text{mix}} \sqcup D_{\text{ft}}$, where
$D_{\text{mix}}$ is used to privately learn the mixture weights $\lambda$, and
$D_{\text{ft}}$ is used for DP-SGD fine-tuning. Because the two mechanisms operate on
disjoint subsets, parallel composition gives an end-to-end guarantee of
$(\epsilon,\delta)$-DP. In our experiments, this is $(\epsilon,\delta) = (1, 10^{-5})$.

We assume a deliberately strong adversary who observes everything released by the
pipeline, including all public data, the algorithm and its hyperparameters, the learned
mixture weights $\lambda$, and the final model. The adversary additionally knows every
private example except the target example whose membership it seeks to infer.

\paragraph{Audit scope and methodology.}
We audit the component novel to DP-MixMin, $M_{\text{mix}}: D_{\text{mix}} \to \lambda$.
For an example in $D_{\text{mix}}$, all subsequent model outputs depend on that example
only through $\lambda$ and are therefore post-processing of $\lambda$. We consequently
audit the released mixture weights directly.

Following the empirical privacy-auditing framework and algorithms of
\citet{jagielski2020auditing}, we repeatedly run the mechanism on neighboring datasets and
estimate an empirical lower bound $\epsilon_{LB}$. Each audit configuration is evaluated
over 2000 independent runs. We use a mixture-finding split ratio of $\rho = 0.1$.
Public sources are the 14 CheXpert classes (proxy predictions). We run DP-MixMin once per
NIH target class, auditing 14 mixture-finding instances in total.

We evaluate two types of canaries. The \textbf{natural canary} is a real, unmodified
private training example whose membership the adversary tries to infer. This captures
realistic leakage of a genuine data point. The \textbf{max-clip canary} is a strong
construction. Its per-example gradient attains the clipping norm, the maximum influence
any single example can have under clipping, in a fixed direction aligned with the
mixture's most-weighted sources, since the released mixture is most sensitive to
perturbations of those sources.

\paragraph{Results.}
At the claimed privacy level $\epsilon=1$, the natural-canary audit gives
$\epsilon_{LB} = 0$ for all 14 NIH tasks. The stronger max-clipped-gradient canary gives
$\max_{\text{task}} \epsilon_{LB} = 0.131$. Thus, all measured private lower bounds satisfy
$\epsilon_{LB} \leq 0.131 < 1$.

When privacy noise is removed, the same audit obtains values as high as
$\epsilon_{LB} = 5.60$, demonstrating that it can detect substantial leakage when present.
Overall, we find no empirical evidence of a privacy violation. Full per-class results for both CheXpert and MIMIC-CXR pretraining are reported in
Tables~\ref{tab:audit-chexpert} and \ref{tab:audit-mimic}.

\subsection{User-Level DP-Mixmin}
User-level DP protects the contribution of an entire user rather than a single example: two neighboring datasets may differ in all examples contributed by one user~\citep{levy2021learning}. For Enron, we therefore partition the private mixture-learning split by user, 
$D_{\lambda}^{priv}=\bigsqcup_{u=1}^{N}D_u$, where $D_u$ contains the examples attributed to user $u$. At each DP-MixMin step, we sample a batch $\mathcal{U}_j$ of $B$ users, compute one gradient per user by averaging the loss over that user's sampled examples, clip this per-user gradient, and add Gaussian noise before updating the mixture weights. This differs from example-level DP-MixMin, where clipping and noise are applied to individual examples.

The parameter $G_{\mathrm{ULS}}$ is the user-level subsample size. For each sampled user $u$, we use
$k_u=\min(G_{\mathrm{ULS}}, |D_u|)$ examples, so users with many emails cannot dominate a single update, while users with fewer than $G_{\mathrm{ULS}}$ emails contribute all available examples. Because the gradient is averaged within each user before clipping, the sensitivity is controlled at the user level: the clipped object is the user's aggregate contribution, not an individual email.

\begin{algorithm}
\caption{\textcolor{blue}{User-Level DP}-MixMin}
\label{alg:user-wise-dp-mixmin}
\begin{algorithmic}[1]
\REQUIRE \textcolor{blue}{Target privacy guarantees $(\epsilon,\delta)$}, step size $\eta$, number of steps $n$, loss function $L$ (CE or MSE), private mixture-learning split $D_{\lambda}^{priv} \subseteq \mathcal{X} \times \mathcal{Y}$ \textcolor{blue}{partitioned by user as $D_{\lambda}^{priv} = \bigsqcup_{u=1}^{N} D_u$}, proxy models $\{\hat{f}_i : i = 1, \ldots, k\}$, \textcolor{blue}{user batch size $B$, user-level subsample size $G_{\mathrm{ULS}}$, clipping norm $C$}
\STATE Initialize $\lambda_i \leftarrow \tfrac{1}{k}$ for all $i = 1, \ldots, k$
\STATE Pre-compute $\{\hat{f}_i(x) : x \in D_{\lambda}^{priv}, \, i = 1, \ldots, k\}$
\FOR{$j = 1$ to $n$}
    \STATE \textcolor{blue}{Sample a batch $\mathcal{U}_j$ of $B$ users from $\{1, \ldots, N\}$}
    \STATE $\hat{f}_\lambda(x) \leftarrow \sum_{i=1}^{k} \lambda_i \hat{f}_i(x)$
    \FOR{\textcolor{blue}{$u \in \mathcal{U}_j$}}
        \STATE \textcolor{blue}{$k_u \leftarrow \min(G_{\mathrm{ULS}}, |D_u|)$}
        \STATE \textcolor{blue}{Sample $D_u^{(j)} \subseteq D_u$ with $|D_u^{(j)}| = k_u$}
        \STATE \textcolor{blue}{$\ell_u \leftarrow \tfrac{1}{k_u} \sum_{(x,y) \in D_u^{(j)}} L(\hat{f}_\lambda(x), y)$}
        \STATE \textcolor{blue}{$g_u \leftarrow \nabla_\lambda \ell_u$}
        \STATE \textcolor{blue}{$\tilde{g}_u \leftarrow \text{Clip}(g_u, C)$} \COMMENT{\textcolor{blue}{clip per-user gradients}}
    \ENDFOR
    \STATE $g_{\text{sum}} \leftarrow \textcolor{blue}{\sum_{u \in \mathcal{U}_j} \tilde{g}_u}$
    \STATE \textcolor{blue}{$z \sim \mathcal{N}(0, \sigma^2 I)$ \COMMENT{noise calibrated to $(\epsilon, \delta)$}}
    \STATE $\bar{g} \leftarrow \tfrac{1}{\textcolor{blue}{B}} (g_{\text{sum}} + z)$
    \STATE $\lambda_i \leftarrow \lambda_i \, e^{-\eta \bar{g}_i}$ for all $i = 1, \ldots, k$
    \STATE $\lambda \leftarrow \lambda \,/\, \sum_{i=1}^{k} \lambda_i$
\ENDFOR
\STATE Return $\lambda$
\end{algorithmic}
\end{algorithm}

\subsection{Random Search}
\label{app:rand_search}

Random Search is the standard hyperparameter optimization baseline across fields~\citep{bergstra2012random,papernot2022hyperparameter,thudi2026mixminfindingdatamixtures}: rather than searching the space of configurations in a structured way, one simply draws a fixed number of candidates independently at random, evaluates each on a held-out validation set, and returns the best performer. Despite its simplicity, Random Search is known to be competitive with, and often outperform, grid search and many model-based optimizers, particularly in moderate-to-high-dimensional spaces where only a few hyperparameters meaningfully affect performance~\citep{bergstra2012random}. Note that private hyperparameter search algorithms also do random search, but randomize the number of candidates to test~\citep{papernot2022hyperparameter}.

We adapted this baseline to mixture selection as follows. We sampled $m$ candidate mixture weights $\lambda_1, \dots, \lambda_m$ independently from the probability simplex $\Delta_k$, ran mixture-weighted public pretraining followed by downstream DP fine-tuning for each $\lambda_i$, and selected the best-performing candidate on the same held-out set $D_\lambda^{\text{priv}}$ that DP-MixMin used for its mixture-selection step. We controlled compute by choosing $m$ so that the total cost of training and evaluating all $m$ mixtures matched or exceeded the end-to-end compute of DP-MixMin, ensuring any gains over Random Search could not be attributed to a larger compute budget. The selection step itself was also implemented non-privately: our implementation of Random Search read $D_\lambda^{\text{priv}}$ directly when picking the winning candidate, paying no privacy cost for the selection step. Random Search was therefore a strong baseline along two axes: equal or greater compute, and no utility degradation from private selection.

\subsection{Tables and Figures}
\begin{table}[t]
  \caption{\textbf{Per-class final AUC on NIH ChestX-ray14 after transfer from CheXpert.} Each entry is the mean $\pm$ standard deviation of AUC over 3--5 seeds at $\delta = 10^{-5}$. Bold indicates the best method per (class, $\epsilon$) cell. DP-MixMin achieves the best macro AUC across all three privacy budgets, with absolute gains of $+0.029$ to $+0.037$ over uniform pretraining.}
  \label{tab:chexpert_to_nih_full}
  \centering
  \scriptsize
  \setlength{\tabcolsep}{3pt}
  \resizebox{\textwidth}{!}{%
  \begin{tabular}{l ccc ccc ccc ccc}
    \toprule
    & \multicolumn{3}{c}{\textbf{DP-MixMin}} & \multicolumn{3}{c}{\textbf{Uniform}} & \multicolumn{3}{c}{\textbf{Random Search}} & \multicolumn{3}{c}{\textbf{No Pretraining}} \\
    \cmidrule(lr){2-4} \cmidrule(lr){5-7} \cmidrule(lr){8-10} \cmidrule(lr){11-13}
    \textbf{Class} & $\epsilon{=}1$ & $\epsilon{=}2$ & $\epsilon{=}5$ & $\epsilon{=}1$ & $\epsilon{=}2$ & $\epsilon{=}5$ & $\epsilon{=}1$ & $\epsilon{=}2$ & $\epsilon{=}5$ & $\epsilon{=}1$ & $\epsilon{=}2$ & $\epsilon{=}5$ \\
    \midrule
    Atelectasis      & \textbf{0.740$\pm$0.012} & \textbf{0.747$\pm$0.005} & \textbf{0.747$\pm$0.007} & 0.707$\pm$0.013 & 0.724$\pm$0.009 & 0.736$\pm$0.008 & 0.724$\pm$0.009 & 0.733$\pm$0.011 & 0.739$\pm$0.010 & 0.679$\pm$0.006 & 0.696$\pm$0.008 & 0.715$\pm$0.006 \\
    Consolidation    & 0.742$\pm$0.014 & 0.745$\pm$0.017 & 0.748$\pm$0.014 & \textbf{0.745$\pm$0.015} & \textbf{0.749$\pm$0.014} & \textbf{0.752$\pm$0.013} & 0.734$\pm$0.008 & 0.740$\pm$0.011 & 0.742$\pm$0.010 & 0.710$\pm$0.006 & 0.716$\pm$0.010 & 0.725$\pm$0.010 \\
    Infiltration     & \textbf{0.671$\pm$0.004} & \textbf{0.674$\pm$0.005} & \textbf{0.678$\pm$0.002} & 0.668$\pm$0.004 & 0.670$\pm$0.003 & 0.672$\pm$0.003 & 0.664$\pm$0.004 & 0.670$\pm$0.003 & 0.675$\pm$0.003 & 0.663$\pm$0.004 & 0.672$\pm$0.003 & 0.680$\pm$0.002 \\
    Pneumothorax     & \textbf{0.780$\pm$0.007} & \textbf{0.788$\pm$0.002} & \textbf{0.798$\pm$0.004} & 0.697$\pm$0.008 & 0.728$\pm$0.008 & 0.753$\pm$0.009 & 0.674$\pm$0.005 & 0.701$\pm$0.009 & 0.726$\pm$0.007 & 0.650$\pm$0.009 & 0.684$\pm$0.005 & 0.723$\pm$0.005 \\
    Edema            & 0.829$\pm$0.017 & 0.836$\pm$0.014 & 0.845$\pm$0.013 & 0.829$\pm$0.006 & 0.835$\pm$0.007 & 0.838$\pm$0.008 & \textbf{0.844$\pm$0.008} & \textbf{0.850$\pm$0.008} & \textbf{0.857$\pm$0.009} & 0.804$\pm$0.012 & 0.816$\pm$0.010 & 0.829$\pm$0.010 \\
    Emphysema        & \textbf{0.683$\pm$0.028} & \textbf{0.715$\pm$0.022} & \textbf{0.725$\pm$0.016} & 0.632$\pm$0.015 & 0.652$\pm$0.010 & 0.688$\pm$0.016 & 0.608$\pm$0.008 & 0.637$\pm$0.006 & 0.664$\pm$0.009 & 0.566$\pm$0.037 & 0.592$\pm$0.027 & 0.631$\pm$0.013 \\
    Fibrosis         & \textbf{0.627$\pm$0.039} & \textbf{0.660$\pm$0.016} & \textbf{0.668$\pm$0.021} & 0.609$\pm$0.022 & 0.641$\pm$0.029 & 0.658$\pm$0.024 & 0.591$\pm$0.019 & 0.618$\pm$0.013 & 0.632$\pm$0.012 & 0.578$\pm$0.038 & 0.612$\pm$0.023 & 0.653$\pm$0.013 \\
    Effusion         & \textbf{0.848$\pm$0.007} & \textbf{0.850$\pm$0.002} & \textbf{0.852$\pm$0.006} & 0.819$\pm$0.006 & 0.830$\pm$0.009 & 0.833$\pm$0.007 & 0.823$\pm$0.009 & 0.835$\pm$0.003 & 0.839$\pm$0.003 & 0.759$\pm$0.004 & 0.784$\pm$0.004 & 0.807$\pm$0.002 \\
    Pneumonia        & 0.626$\pm$0.035 & 0.626$\pm$0.025 & 0.650$\pm$0.022 & 0.649$\pm$0.023 & \textbf{0.660$\pm$0.032} & \textbf{0.661$\pm$0.030} & 0.648$\pm$0.025 & 0.652$\pm$0.039 & \textbf{0.663$\pm$0.045} & 0.613$\pm$0.021 & 0.634$\pm$0.019 & 0.642$\pm$0.019 \\
    Pleural Thick.   & \textbf{0.655$\pm$0.008} & \textbf{0.668$\pm$0.009} & \textbf{0.676$\pm$0.027} & 0.639$\pm$0.027 & 0.652$\pm$0.022 & 0.670$\pm$0.017 & 0.604$\pm$0.009 & 0.629$\pm$0.009 & 0.643$\pm$0.009 & 0.582$\pm$0.014 & 0.593$\pm$0.010 & 0.610$\pm$0.013 \\
    Cardiomegaly     & \textbf{0.872$\pm$0.008} & \textbf{0.878$\pm$0.010} & \textbf{0.883$\pm$0.010} & 0.710$\pm$0.009 & 0.733$\pm$0.006 & 0.748$\pm$0.010 & 0.762$\pm$0.004 & 0.792$\pm$0.002 & 0.820$\pm$0.004 & 0.669$\pm$0.004 & 0.679$\pm$0.010 & 0.691$\pm$0.010 \\
    Nodule           & \textbf{0.574$\pm$0.009} & 0.566$\pm$0.017 & 0.584$\pm$0.015 & 0.550$\pm$0.009 & 0.564$\pm$0.007 & 0.582$\pm$0.004 & 0.546$\pm$0.022 & 0.564$\pm$0.018 & 0.583$\pm$0.017 & 0.563$\pm$0.017 & \textbf{0.578$\pm$0.013} & \textbf{0.584$\pm$0.015} \\
    Mass             & \textbf{0.674$\pm$0.009} & \textbf{0.688$\pm$0.009} & \textbf{0.700$\pm$0.012} & 0.631$\pm$0.010 & 0.648$\pm$0.009 & 0.666$\pm$0.005 & 0.619$\pm$0.009 & 0.634$\pm$0.006 & 0.652$\pm$0.007 & 0.550$\pm$0.023 & 0.571$\pm$0.018 & 0.593$\pm$0.007 \\
    Hernia           & \textbf{0.583$\pm$0.054} & \textbf{0.624$\pm$0.054} & \textbf{0.655$\pm$0.044} & 0.499$\pm$0.043 & 0.531$\pm$0.042 & 0.551$\pm$0.048 & 0.498$\pm$0.036 & 0.523$\pm$0.027 & 0.538$\pm$0.045 & 0.517$\pm$0.022 & 0.551$\pm$0.050 & 0.561$\pm$0.047 \\
    \midrule
\textbf{Macro AUC}             & \textbf{0.707} & \textbf{0.719} & \textbf{0.729} & 0.670 & 0.687 & 0.701 & 0.667 & 0.684 & 0.698 & 0.636 & 0.656 & 0.675 \\
\textbf{$\Delta$ vs.\ Uniform} & $+0.037$       & $+0.032$       & $+0.029$       & ---   & ---   & ---   & $-0.003$ & $-0.003$ & $-0.003$ & $-0.034$ & $-0.031$ & $-0.026$ \\
    \bottomrule
  \end{tabular}%
  }
\end{table}

\begin{table}[t]
  \caption{\textbf{Per-class final AUC on NIH ChestX-ray14 after transfer from MIMIC-CXR.} Each entry is the mean $\pm$ standard deviation of AUC over 3--5 seeds at $\delta = 10^{-5}$. Bold indicates the best method per (class, $\epsilon$) cell. DP-MixMin achieves the best macro AUC across all three privacy budgets, with absolute gains of $+0.024$ to $+0.033$ over uniform pretraining.}
  \label{tab:mimic_to_nih_full}
  \centering
  \scriptsize
  \setlength{\tabcolsep}{3pt}
  \resizebox{\textwidth}{!}{%
  \begin{tabular}{l ccc ccc ccc ccc}
    \toprule
    & \multicolumn{3}{c}{\textbf{DP-MixMin}} & \multicolumn{3}{c}{\textbf{Uniform}} & \multicolumn{3}{c}{\textbf{Random Search}} & \multicolumn{3}{c}{\textbf{No Pretraining}} \\
    \cmidrule(lr){2-4} \cmidrule(lr){5-7} \cmidrule(lr){8-10} \cmidrule(lr){11-13}
    \textbf{Class} & $\epsilon{=}1$ & $\epsilon{=}2$ & $\epsilon{=}5$ & $\epsilon{=}1$ & $\epsilon{=}2$ & $\epsilon{=}5$ & $\epsilon{=}1$ & $\epsilon{=}2$ & $\epsilon{=}5$ & $\epsilon{=}1$ & $\epsilon{=}2$ & $\epsilon{=}5$ \\
    \midrule
    Atelectasis      & \textbf{0.751$\pm$0.010} & 0.749$\pm$0.010 & 0.737$\pm$0.000 & 0.728$\pm$0.009 & 0.734$\pm$0.008 & 0.740$\pm$0.008 & 0.742$\pm$0.006 & \textbf{0.749$\pm$0.006} & \textbf{0.755$\pm$0.005} & 0.679$\pm$0.006 & 0.696$\pm$0.008 & 0.715$\pm$0.006 \\
    Consolidation    & \textbf{0.756$\pm$0.012} & \textbf{0.758$\pm$0.011} & \textbf{0.759$\pm$0.015} & 0.748$\pm$0.010 & 0.751$\pm$0.011 & 0.753$\pm$0.011 & 0.743$\pm$0.012 & 0.746$\pm$0.012 & 0.749$\pm$0.012 & 0.710$\pm$0.006 & 0.716$\pm$0.010 & 0.725$\pm$0.010 \\
    Infiltration     & \textbf{0.674$\pm$0.008} & \textbf{0.679$\pm$0.004} & \textbf{0.682$\pm$0.005} & 0.667$\pm$0.004 & 0.673$\pm$0.005 & 0.677$\pm$0.006 & 0.673$\pm$0.006 & 0.676$\pm$0.006 & 0.680$\pm$0.003 & 0.663$\pm$0.004 & 0.672$\pm$0.003 & 0.680$\pm$0.002 \\
    Pneumothorax     & \textbf{0.754$\pm$0.013} & \textbf{0.770$\pm$0.004} & \textbf{0.782$\pm$0.005} & 0.719$\pm$0.012 & 0.741$\pm$0.015 & 0.754$\pm$0.015 & 0.679$\pm$0.015 & 0.707$\pm$0.016 & 0.732$\pm$0.016 & 0.650$\pm$0.009 & 0.684$\pm$0.005 & 0.723$\pm$0.005 \\
    Edema            & \textbf{0.845$\pm$0.003} & \textbf{0.849$\pm$0.004} & \textbf{0.853$\pm$0.005} & 0.839$\pm$0.007 & 0.841$\pm$0.008 & 0.845$\pm$0.007 & 0.826$\pm$0.011 & 0.834$\pm$0.008 & 0.841$\pm$0.011 & 0.804$\pm$0.012 & 0.816$\pm$0.010 & 0.829$\pm$0.010 \\
    Emphysema        & \textbf{0.710$\pm$0.023} & \textbf{0.728$\pm$0.012} & \textbf{0.742$\pm$0.012} & 0.650$\pm$0.013 & 0.674$\pm$0.010 & 0.691$\pm$0.011 & 0.633$\pm$0.033 & 0.645$\pm$0.022 & 0.693$\pm$0.034 & 0.566$\pm$0.037 & 0.592$\pm$0.027 & 0.631$\pm$0.013 \\
    Fibrosis         & \textbf{0.641$\pm$0.019} & \textbf{0.653$\pm$0.019} & \textbf{0.680$\pm$0.028} & 0.614$\pm$0.024 & 0.638$\pm$0.026 & 0.637$\pm$0.018 & 0.587$\pm$0.006 & 0.620$\pm$0.004 & 0.644$\pm$0.010 & 0.578$\pm$0.038 & 0.612$\pm$0.023 & 0.653$\pm$0.013 \\
    Effusion         & \textbf{0.849$\pm$0.003} & \textbf{0.849$\pm$0.005} & \textbf{0.851$\pm$0.005} & 0.824$\pm$0.007 & 0.830$\pm$0.005 & 0.833$\pm$0.006 & 0.813$\pm$0.003 & 0.822$\pm$0.004 & 0.829$\pm$0.006 & 0.759$\pm$0.004 & 0.784$\pm$0.004 & 0.807$\pm$0.002 \\
    Pneumonia        & \textbf{0.681$\pm$0.023} & \textbf{0.677$\pm$0.027} & 0.674$\pm$0.020 & 0.666$\pm$0.029 & 0.664$\pm$0.041 & \textbf{0.682$\pm$0.031} & 0.645$\pm$0.029 & 0.655$\pm$0.029 & 0.667$\pm$0.029 & 0.613$\pm$0.021 & 0.634$\pm$0.019 & 0.642$\pm$0.019 \\
    Pleural Thick.   & \textbf{0.683$\pm$0.013} & \textbf{0.688$\pm$0.017} & \textbf{0.704$\pm$0.022} & 0.664$\pm$0.014 & 0.676$\pm$0.011 & 0.684$\pm$0.010 & 0.621$\pm$0.013 & 0.641$\pm$0.005 & 0.663$\pm$0.012 & 0.582$\pm$0.014 & 0.593$\pm$0.010 & 0.610$\pm$0.013 \\
    Cardiomegaly     & \textbf{0.804$\pm$0.025} & \textbf{0.816$\pm$0.016} & \textbf{0.835$\pm$0.009} & 0.711$\pm$0.007 & 0.765$\pm$0.026 & 0.782$\pm$0.027 & 0.749$\pm$0.012 & 0.770$\pm$0.019 & 0.801$\pm$0.017 & 0.669$\pm$0.004 & 0.679$\pm$0.010 & 0.691$\pm$0.010 \\
    Nodule           & \textbf{0.589$\pm$0.012} & \textbf{0.591$\pm$0.016} & \textbf{0.598$\pm$0.008} & 0.552$\pm$0.013 & 0.569$\pm$0.012 & 0.585$\pm$0.009 & 0.539$\pm$0.008 & 0.548$\pm$0.009 & 0.559$\pm$0.010 & 0.563$\pm$0.017 & 0.578$\pm$0.013 & 0.584$\pm$0.015 \\
    Mass             & \textbf{0.690$\pm$0.008} & \textbf{0.707$\pm$0.005} & \textbf{0.719$\pm$0.009} & 0.635$\pm$0.006 & 0.654$\pm$0.004 & 0.667$\pm$0.007 & 0.597$\pm$0.014 & 0.628$\pm$0.017 & 0.653$\pm$0.017 & 0.550$\pm$0.023 & 0.571$\pm$0.018 & 0.593$\pm$0.007 \\
    Hernia           & \textbf{0.546$\pm$0.048} & \textbf{0.620$\pm$0.003} & \textbf{0.604$\pm$0.058} & 0.495$\pm$0.024 & 0.536$\pm$0.053 & 0.551$\pm$0.033 & 0.527$\pm$0.056 & 0.525$\pm$0.040 & 0.580$\pm$0.053 & 0.517$\pm$0.022 & 0.551$\pm$0.050 & 0.561$\pm$0.047 \\
    \midrule
\textbf{Macro AUC}             & \textbf{0.712} & \textbf{0.724} & \textbf{0.730} & 0.679 & 0.696 & 0.706 & 0.670 & 0.683 & 0.703 & 0.636 & 0.656 & 0.675 \\
    \textbf{$\Delta$ vs.\ Uniform} & $+0.033$       & $+0.028$       & $+0.024$       & ---   & ---   & ---   & $-0.010$ & $-0.013$ & $-0.003$ & $-0.044$ & $-0.041$ & $-0.031$ \\
    \bottomrule

  \end{tabular}%
  }
\end{table}

\begin{figure}[t]
  \centering
  \includegraphics[width=0.65\linewidth]{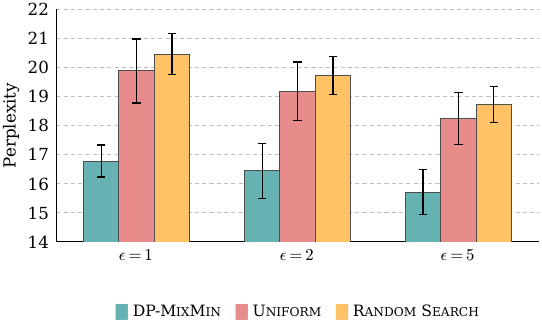}
  \caption{\textbf{DP-MixMin beats both uniform-mixture and random-search Common Pile pretraining from scratch on Enron email completion across all three privacy budgets, with the largest gain at the tightest budget.} Test-set perplexity after DP fine-tuning on Enron ($\delta{=}10^{-5}$) at $\epsilon \in \{1, 2, 5\}$. All three methods pretrain GPT-2 medium from random initialization on Common Pile and are then DP-finetuned on Enron under identical settings. DP-MixMin reduces test perplexity by roughly $16\%/14\%/14\%$ over uniform and $18\%/17\%/16\%$ over random-mixture pretraining at $\epsilon \in \{1,2,5\}$.}
  \label{fig:enron_perplexity_comparison}
\end{figure}

\begin{figure}[t]
  \centering
  \includegraphics[width=0.65\linewidth]{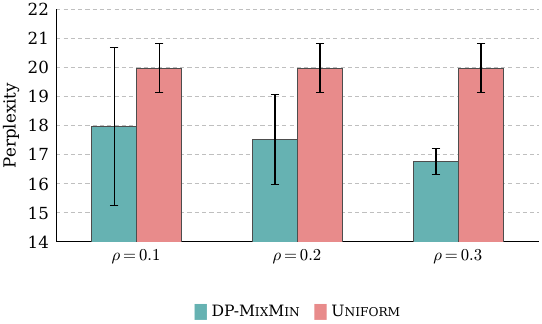}
  \caption{\textbf{DP-MixMin consistently outperforms uniform Common Pile pretraining on Enron at $\epsilon{=}1$ across all proxy-split fractions, with larger fractions yielding both lower perplexity and tighter cross-seed variance.} Test-set perplexity after DP fine-tuning on Enron ($\delta{=}10^{-5}$) at $\epsilon{=}1$, varying the proxy fraction $\rho \in \{0.1, 0.2, 0.3\}$ used by MixMin to estimate domain weights. The uniform baseline is independent of $\rho$ and is shown for reference at all three values. DP-MixMin reduces test perplexity by $10\%/12\%/16\%$ over uniform at $\rho \in \{0.1, 0.2, 0.3\}$, with $\rho{=}0.3$ giving the most stable mixture (cross-seed $\sigma$ drops from $2.72 \to 0.45$).}
  \label{fig:enron_perc_ablation}
\end{figure}

\begin{table}[t]
  \caption{\textbf{NIH ChestX-ray14 class distribution and correspondence to CheXpert/MIMIC-CXR labels.} The \textbf{Match} column indicates whether each NIH class has a direct (\emph{Yes}), \emph{Partial}, or no (\emph{None}) equivalent in CheXpert/MIMIC-CXR.}
  \label{tab:nih-chexpert-mapping}
  \centering
  \small
  \setlength{\tabcolsep}{5pt}
  \begin{tabular}{@{}c l r r c l@{}}
    \toprule
    \textbf{NIH Idx} & \textbf{Class Name} & \textbf{\# Pos.} & \textbf{\% Pos.} & \textbf{Match} & \textbf{CheXpert/MIMIC-CXR Label} \\
    \midrule
    0  & Atelectasis        & 11{,}559 & 10.31\% & Yes     & Atelectasis \\
    1  & Consolidation      &  4{,}667 &  4.16\% & Yes     & Consolidation \\
    2  & Infiltration       & 19{,}894 & 17.75\% & Partial & Lung Opacity \\
    3  & Pneumothorax       &  5{,}302 &  4.73\% & Yes     & Pneumothorax \\
    4  & Edema              &  2{,}303 &  2.05\% & Yes     & Edema \\
    5  & Emphysema          &  2{,}516 &  2.24\% & None    & --- \\
    6  & Fibrosis           &  1{,}686 &  1.50\% & None    & --- \\
    7  & Effusion           & 13{,}317 & 11.88\% & Yes     & Pleural Effusion \\
    8  & Pneumonia          &  1{,}431 &  1.28\% & Yes     & Pneumonia \\
    9  & Pleural Thickening &  3{,}385 &  3.02\% & Partial & Pleural Other \\
    10 & Cardiomegaly       &  2{,}776 &  2.48\% & Yes     & Cardiomegaly \\
    11 & Nodule             &  6{,}331 &  5.65\% & Partial & Lung Lesion \\
    12 & Mass               &  5{,}782 &  5.16\% & Partial & Lung Lesion \\
    13 & Hernia             &     227  &  0.20\% & None    & --- \\
    \bottomrule
  \end{tabular}
\end{table}

\begin{table}[t]
  \caption{\textbf{Average AUC gains of DP-MixMin over Uniform pretraining, grouped by class type.} Group~A: well-represented classes with a direct CheXpert match (Atelectasis, Cardiomegaly, Effusion, Pneumonia, Pneumothorax, Consolidation, Edema). Group~B: rare or unmatched classes (Infiltration, Mass, Nodule, Emphysema, Fibrosis, Pleural Thickening, Hernia); see Table~\ref{tab:nih-chexpert-mapping}. Each cell is the within-group mean of per-class absolute (abs) and relative (rel) AUC gain of DP-MixMin over Uniform pretraining at $\delta = 10^{-5}$. }
  \label{tab:group-gains}
  \centering
  \small
  \begin{tabular}{@{}l cc cc cc@{}}
    \toprule
    & \multicolumn{2}{c}{$\epsilon = 1$}
    & \multicolumn{2}{c}{$\epsilon = 2$}
    & \multicolumn{2}{c}{$\epsilon = 5$} \\
    \cmidrule(lr){2-3} \cmidrule(lr){4-5} \cmidrule(lr){6-7}
    & abs & rel & abs & rel & abs & rel \\
    \midrule
    \multicolumn{7}{@{}l}{\textit{CheXpert $\to$ NIH ChestX-ray14}} \\
    Group A ($n=7$) & $+0.040$ & $+5.6\%$ & $+0.030$ & $+4.0\%$ & $+0.029$ & $+3.8\%$ \\
    Group B ($n=7$) & $+0.034$ & $+6.0\%$ & $+0.034$ & $+5.7\%$ & $+0.028$ & $+4.7\%$ \\
    \midrule
    \multicolumn{7}{@{}l}{\textit{MIMIC-CXR $\to$ NIH ChestX-ray14}} \\
    Group A ($n=7$) & $+0.029$ & $+4.0\%$ & $+0.020$ & $+2.7\%$ & $+0.015$ & $+1.8\%$ \\
    Group B ($n=7$) & $+0.037$ & $+6.2\%$ & $+0.035$ & $+5.8\%$ & $+0.034$ & $+5.3\%$ \\
    \bottomrule
  \end{tabular}
\end{table}

\begin{table}[t]
  \caption{\textbf{Weight reconstruction of DP-MixMin across 14 NIH ChestX-ray14 tasks.} Mean $\pm$ std over $n{=}70$ paired comparisons (14 tasks $\times$ 5 seeds). We compare the private weights $\boldsymbol{\lambda}_{\mathrm{DP}}$ to the non-private MixMin solution $\boldsymbol{\lambda}^{\star}$ using cosine similarity $\mathrm{cos}_\star$ and $L_1$ distance. Uniform references are CheXpert: $\mathrm{cos}_{\mathbf{u}}{=}0.75$, $L_{1,\mathbf{u}}{=}0.68$; MIMIC-CXR: $\mathrm{cos}_{\mathbf{u}}{=}0.80$, $L_{1,\mathbf{u}}{=}0.60$. The DP-MixMin selection budget $\rho = 0.1$.  Bold values improve over the corresponding uniform reference.}
  \label{tab:vision-weights-all}
  \centering
  \scriptsize
  \setlength{\tabcolsep}{3pt}
  \resizebox{\linewidth}{!}{%
  \begin{tabular}{@{}l c cc cc cc cc cc@{}}
    \toprule
    & & \multicolumn{2}{c}{$\epsilon = 0.2$}
    & \multicolumn{2}{c}{$\epsilon = 0.5$}
    & \multicolumn{2}{c}{$\epsilon = 1$}
    & \multicolumn{2}{c}{$\epsilon = 2$}
    & \multicolumn{2}{c}{$\epsilon = 5$} \\
    \cmidrule(lr){3-4} \cmidrule(lr){5-6} \cmidrule(lr){7-8} \cmidrule(lr){9-10} \cmidrule(lr){11-12}
    src & $n$
    & $\mathrm{cos}_\star$ & $L_1$
    & $\mathrm{cos}_\star$ & $L_1$
    & $\mathrm{cos}_\star$ & $L_1$
    & $\mathrm{cos}_\star$ & $L_1$
    & $\mathrm{cos}_\star$ & $L_1$ \\
    \midrule
    \multicolumn{12}{@{}l}{\textit{CheXpert $\to$ NIH ChestX-ray14}} \\
    CheXpert & $70$
    & $0.70{\scriptstyle\pm.08}$ & $0.81{\scriptstyle\pm.14}$
    & $\mathbf{0.93}{\scriptstyle\pm.03}$ & $\mathbf{0.37}{\scriptstyle\pm.09}$
    & $\mathbf{0.97}{\scriptstyle\pm.01}$ & $\mathbf{0.24}{\scriptstyle\pm.06}$
    & $\mathbf{0.98}{\scriptstyle\pm.01}$ & $\mathbf{0.19}{\scriptstyle\pm.05}$
    & $\mathbf{0.99}{\scriptstyle\pm.01}$ & $\mathbf{0.16}{\scriptstyle\pm.04}$ \\
    \midrule
    \multicolumn{12}{@{}l}{\textit{MIMIC-CXR $\to$ NIH ChestX-ray14}} \\
    MIMIC-CXR & $70$
    & $0.58{\scriptstyle\pm.17}$ & $1.04{\scriptstyle\pm.30}$
    & $\mathbf{0.82}{\scriptstyle\pm.09}$ & $\mathbf{0.57}{\scriptstyle\pm.20}$
    & $\mathbf{0.93}{\scriptstyle\pm.04}$ & $\mathbf{0.33}{\scriptstyle\pm.10}$
    & $\mathbf{0.97}{\scriptstyle\pm.01}$ & $\mathbf{0.22}{\scriptstyle\pm.05}$
    & $\mathbf{0.98}{\scriptstyle\pm.01}$ & $\mathbf{0.16}{\scriptstyle\pm.03}$ \\
    \bottomrule
  \end{tabular}%
  }
\end{table}

\begin{table}[t]
  \caption{\textbf{Weight reconstruction of DP-MixMin for NIH Hernia} (0.20\% prevalence in $D^{\mathrm{priv}}$, 227 total positives). Mean $\pm$ std over 5 paired comparisons (5 seeds). We compare the private weights $\boldsymbol{\lambda}_{\mathrm{DP}}$ to the non-private MixMin solution $\boldsymbol{\lambda}^{\star}$ using cosine similarity $\mathrm{cos}_\star$ and $L_1$ distance. Uniform references: CheXpert $\mathrm{cos}_{\mathbf{u}}{=}0.64$, $L_{1,\mathbf{u}}{=}0.96$; MIMIC-CXR $\mathrm{cos}_{\mathbf{u}}{=}0.73$, $L_{1,\mathbf{u}}{=}0.82$. Bold values improve over the corresponding uniform reference.}
  \label{tab:weights-hernia}
  \centering
  \scriptsize
  \setlength{\tabcolsep}{4pt}
  \begin{tabular}{@{}l c cc cc cc@{}}
    \toprule
    & & \multicolumn{2}{c}{$\rho{=}0.05$} & \multicolumn{2}{c}{$\rho{=}0.1$} & \multicolumn{2}{c}{$\rho{=}0.2$} \\
    & & \multicolumn{2}{c}{\scriptsize $n{\approx}4{,}485$, \#pos${\approx}9$} & \multicolumn{2}{c}{\scriptsize $n{\approx}8{,}970$, \#pos${\approx}18$} & \multicolumn{2}{c}{\scriptsize $n{\approx}17{,}939$, \#pos${\approx}36$} \\
    \cmidrule(lr){3-4} \cmidrule(lr){5-6} \cmidrule(lr){7-8}
    Pretrain & $\varepsilon$ & $\mathrm{cos}_\star$ & $L_1$ & $\mathrm{cos}_\star$ & $L_1$ & $\mathrm{cos}_\star$ & $L_1$ \\
    \midrule
    CheXpert  & $0.2$ & $\mathbf{0.79}{\scriptstyle\pm0.01}$ & $\mathbf{0.64}{\scriptstyle\pm0.02}$ & $\mathbf{0.94}{\scriptstyle\pm0.01}$ & $\mathbf{0.32}{\scriptstyle\pm0.01}$ & $\mathbf{0.98}{\scriptstyle\pm0.00}$ & $\mathbf{0.17}{\scriptstyle\pm0.01}$ \\
              & $0.5$ & $\mathbf{0.96}{\scriptstyle\pm0.01}$ & $\mathbf{0.27}{\scriptstyle\pm0.03}$ & $\mathbf{0.99}{\scriptstyle\pm0.00}$ & $\mathbf{0.15}{\scriptstyle\pm0.01}$ & $\mathbf{0.99}{\scriptstyle\pm0.00}$ & $\mathbf{0.10}{\scriptstyle\pm0.01}$ \\
              & $1.0$ & $\mathbf{0.99}{\scriptstyle\pm0.00}$ & $\mathbf{0.16}{\scriptstyle\pm0.03}$ & $\mathbf{0.99}{\scriptstyle\pm0.00}$ & $\mathbf{0.10}{\scriptstyle\pm0.01}$ & $\mathbf{1.00}{\scriptstyle\pm0.00}$ & $\mathbf{0.08}{\scriptstyle\pm0.01}$ \\
              & $2.0$ & $\mathbf{0.99}{\scriptstyle\pm0.00}$ & $\mathbf{0.11}{\scriptstyle\pm0.02}$ & $\mathbf{1.00}{\scriptstyle\pm0.00}$ & $\mathbf{0.08}{\scriptstyle\pm0.01}$ & $\mathbf{1.00}{\scriptstyle\pm0.00}$ & $\mathbf{0.08}{\scriptstyle\pm0.01}$ \\
              & $5.0$ & $\mathbf{1.00}{\scriptstyle\pm0.00}$ & $\mathbf{0.09}{\scriptstyle\pm0.02}$ & $\mathbf{1.00}{\scriptstyle\pm0.00}$ & $\mathbf{0.08}{\scriptstyle\pm0.02}$ & $\mathbf{1.00}{\scriptstyle\pm0.00}$ & $\mathbf{0.08}{\scriptstyle\pm0.01}$ \\
    \midrule
    MIMIC-CXR & $0.2$ & $0.62{\scriptstyle\pm0.00}$ & $1.08{\scriptstyle\pm0.00}$ & $\mathbf{0.79}{\scriptstyle\pm0.00}$ & $\mathbf{0.67}{\scriptstyle\pm0.00}$ & $\mathbf{0.92}{\scriptstyle\pm0.00}$ & $\mathbf{0.35}{\scriptstyle\pm0.00}$ \\
              & $0.5$ & $\mathbf{0.83}{\scriptstyle\pm0.00}$ & $\mathbf{0.57}{\scriptstyle\pm0.00}$ & $\mathbf{0.94}{\scriptstyle\pm0.00}$ & $\mathbf{0.30}{\scriptstyle\pm0.00}$ & $\mathbf{0.98}{\scriptstyle\pm0.00}$ & $\mathbf{0.17}{\scriptstyle\pm0.00}$ \\
              & $1.0$ & $\mathbf{0.94}{\scriptstyle\pm0.00}$ & $\mathbf{0.31}{\scriptstyle\pm0.00}$ & $\mathbf{0.98}{\scriptstyle\pm0.00}$ & $\mathbf{0.17}{\scriptstyle\pm0.00}$ & $\mathbf{0.99}{\scriptstyle\pm0.00}$ & $\mathbf{0.10}{\scriptstyle\pm0.00}$ \\
              & $2.0$ & $\mathbf{0.98}{\scriptstyle\pm0.00}$ & $\mathbf{0.18}{\scriptstyle\pm0.00}$ & $\mathbf{0.99}{\scriptstyle\pm0.00}$ & $\mathbf{0.11}{\scriptstyle\pm0.00}$ & $\mathbf{1.00}{\scriptstyle\pm0.00}$ & $\mathbf{0.08}{\scriptstyle\pm0.00}$ \\
              & $5.0$ & $\mathbf{0.99}{\scriptstyle\pm0.00}$ & $\mathbf{0.10}{\scriptstyle\pm0.00}$ & $\mathbf{1.00}{\scriptstyle\pm0.00}$ & $\mathbf{0.07}{\scriptstyle\pm0.00}$ & $\mathbf{1.00}{\scriptstyle\pm0.00}$ & $\mathbf{0.06}{\scriptstyle\pm0.00}$ \\
    \bottomrule
  \end{tabular}
\end{table}

\begin{table}[t]
  \caption{\textbf{Weight reconstruction of DP-MixMin for NIH Fibrosis} (1.50\% prevalence in $D^{\mathrm{priv}}$, 1{,}686 total positives). Mean $\pm$ std over 5 paired comparisons (5 seeds). We compare the private weights $\boldsymbol{\lambda}_{\mathrm{DP}}$ to the non-private MixMin solution $\boldsymbol{\lambda}^{\star}$ using cosine similarity $\mathrm{cos}_\star$ and $L_1$ distance. Uniform references: CheXpert $\mathrm{cos}_{\mathbf{u}}{=}0.73$, $L_{1,\mathbf{u}}{=}0.77$; MIMIC-CXR $\mathrm{cos}_{\mathbf{u}}{=}0.83$, $L_{1,\mathbf{u}}{=}0.60$. Bold values improve over the corresponding uniform reference.}
  \label{tab:weights-pneumonia}
  \centering
  \scriptsize
  \setlength{\tabcolsep}{4pt}
  \begin{tabular}{@{}l c cc cc cc@{}}
    \toprule
    & & \multicolumn{2}{c}{$\rho{=}0.05$} & \multicolumn{2}{c}{$\rho{=}0.1$} & \multicolumn{2}{c}{$\rho{=}0.2$} \\
    & & \multicolumn{2}{c}{\scriptsize $n{\approx}4{,}485$, \#pos${\approx}57$} & \multicolumn{2}{c}{\scriptsize $n{\approx}8{,}970$, \#pos${\approx}114$} & \multicolumn{2}{c}{\scriptsize $n{\approx}17{,}939$, \#pos${\approx}229$} \\
    \cmidrule(lr){3-4} \cmidrule(lr){5-6} \cmidrule(lr){7-8}
    Pretrain & $\varepsilon$ & $\mathrm{cos}_\star$ & $L_1$ & $\mathrm{cos}_\star$ & $L_1$ & $\mathrm{cos}_\star$ & $L_1$ \\
    \midrule
    CheXpert  & $0.2$ & $\mathbf{0.78}{\scriptstyle\pm0.02}$ & $\mathbf{0.70}{\scriptstyle\pm0.03}$ & $\mathbf{0.94}{\scriptstyle\pm0.01}$ & $\mathbf{0.38}{\scriptstyle\pm0.02}$ & $\mathbf{0.98}{\scriptstyle\pm0.00}$ & $\mathbf{0.23}{\scriptstyle\pm0.01}$ \\
              & $0.5$ & $\mathbf{0.95}{\scriptstyle\pm0.01}$ & $\mathbf{0.33}{\scriptstyle\pm0.04}$ & $\mathbf{0.98}{\scriptstyle\pm0.00}$ & $\mathbf{0.20}{\scriptstyle\pm0.02}$ & $\mathbf{0.99}{\scriptstyle\pm0.00}$ & $\mathbf{0.16}{\scriptstyle\pm0.01}$ \\
              & $1.0$ & $\mathbf{0.98}{\scriptstyle\pm0.01}$ & $\mathbf{0.22}{\scriptstyle\pm0.03}$ & $\mathbf{0.99}{\scriptstyle\pm0.00}$ & $\mathbf{0.16}{\scriptstyle\pm0.01}$ & $\mathbf{0.99}{\scriptstyle\pm0.00}$ & $\mathbf{0.16}{\scriptstyle\pm0.01}$ \\
              & $2.0$ & $\mathbf{0.99}{\scriptstyle\pm0.00}$ & $\mathbf{0.18}{\scriptstyle\pm0.03}$ & $\mathbf{0.99}{\scriptstyle\pm0.00}$ & $\mathbf{0.16}{\scriptstyle\pm0.01}$ & $\mathbf{0.99}{\scriptstyle\pm0.00}$ & $\mathbf{0.15}{\scriptstyle\pm0.01}$ \\
              & $5.0$ & $\mathbf{0.99}{\scriptstyle\pm0.00}$ & $\mathbf{0.17}{\scriptstyle\pm0.02}$ & $\mathbf{0.99}{\scriptstyle\pm0.00}$ & $\mathbf{0.15}{\scriptstyle\pm0.01}$ & $\mathbf{0.99}{\scriptstyle\pm0.00}$ & $\mathbf{0.15}{\scriptstyle\pm0.01}$ \\
    \midrule
    MIMIC-CXR & $0.2$ & $0.59{\scriptstyle\pm0.00}$ & $1.10{\scriptstyle\pm0.01}$ & $0.78{\scriptstyle\pm0.00}$ & $0.69{\scriptstyle\pm0.01}$ & $\mathbf{0.92}{\scriptstyle\pm0.00}$ & $\mathbf{0.37}{\scriptstyle\pm0.01}$ \\
              & $0.5$ & $0.82{\scriptstyle\pm0.00}$ & $0.59{\scriptstyle\pm0.01}$ & $\mathbf{0.94}{\scriptstyle\pm0.00}$ & $\mathbf{0.32}{\scriptstyle\pm0.01}$ & $\mathbf{0.98}{\scriptstyle\pm0.00}$ & $\mathbf{0.19}{\scriptstyle\pm0.00}$ \\
              & $1.0$ & $\mathbf{0.93}{\scriptstyle\pm0.00}$ & $\mathbf{0.34}{\scriptstyle\pm0.01}$ & $\mathbf{0.98}{\scriptstyle\pm0.00}$ & $\mathbf{0.21}{\scriptstyle\pm0.00}$ & $\mathbf{0.99}{\scriptstyle\pm0.00}$ & $\mathbf{0.15}{\scriptstyle\pm0.01}$ \\
              & $2.0$ & $\mathbf{0.97}{\scriptstyle\pm0.00}$ & $\mathbf{0.21}{\scriptstyle\pm0.01}$ & $\mathbf{0.99}{\scriptstyle\pm0.00}$ & $\mathbf{0.16}{\scriptstyle\pm0.00}$ & $\mathbf{0.99}{\scriptstyle\pm0.00}$ & $\mathbf{0.13}{\scriptstyle\pm0.01}$ \\
              & $5.0$ & $\mathbf{0.99}{\scriptstyle\pm0.00}$ & $\mathbf{0.15}{\scriptstyle\pm0.01}$ & $\mathbf{0.99}{\scriptstyle\pm0.00}$ & $\mathbf{0.14}{\scriptstyle\pm0.01}$ & $\mathbf{1.00}{\scriptstyle\pm0.00}$ & $\mathbf{0.10}{\scriptstyle\pm0.01}$ \\
    \bottomrule
  \end{tabular}
\end{table}

\begin{table}[t]
  \caption{\textbf{Enron weight recovery under user-level DP with $K{=}16$ Common Pile proxies.} Mean $\pm$ std over 3 paired comparisons (3 seeds). We compare the private mixture weights $\boldsymbol{\lambda}_{\mathrm{DP}}$ to the non-private MixMin solution $\boldsymbol{\lambda}^{\star}$ using cosine similarity $\mathrm{cos}_\star$ and $L_1$ distance. Columns vary the DP-MixMin selection budget $\rho$, the fraction of the private mixture-learning split used to estimate the mixture weights; the second header row reports the corresponding number of Enron users. The uniform reference is $\mathrm{cos}_{\mathbf{u}}{=}0.40$ and $L_{1,\mathbf{u}}{=}1.47$. Bold values improve over the uniform reference on that metric.}
  \label{tab:enron-weights}
  \centering
  \scriptsize
  \setlength{\tabcolsep}{3pt}
  \resizebox{\linewidth}{!}{%
  \begin{tabular}{@{}l cc cc cc cc cc cc@{}}
    \toprule
    & \multicolumn{2}{c}{$\rho{=}0.05$}
    & \multicolumn{2}{c}{$\rho{=}0.1$}
    & \multicolumn{2}{c}{$\rho{=}0.15$}
    & \multicolumn{2}{c}{$\rho{=}0.2$}
    & \multicolumn{2}{c}{$\rho{=}0.3$}
    & \multicolumn{2}{c}{$\rho{=}0.5$} \\
    & \multicolumn{2}{c}{\scriptsize (6 users)}
    & \multicolumn{2}{c}{\scriptsize (12 users)}
    & \multicolumn{2}{c}{\scriptsize (18 users)}
    & \multicolumn{2}{c}{\scriptsize (24 users)}
    & \multicolumn{2}{c}{\scriptsize (35 users)}
    & \multicolumn{2}{c}{\scriptsize (59 users)} \\
    \cmidrule(lr){2-3} \cmidrule(lr){4-5} \cmidrule(lr){6-7}
    \cmidrule(lr){8-9} \cmidrule(lr){10-11} \cmidrule(lr){12-13}
    $\varepsilon$
    & $\mathrm{cos}_\star$ & $L_1$
    & $\mathrm{cos}_\star$ & $L_1$
    & $\mathrm{cos}_\star$ & $L_1$
    & $\mathrm{cos}_\star$ & $L_1$
    & $\mathrm{cos}_\star$ & $L_1$
    & $\mathrm{cos}_\star$ & $L_1$ \\
    \midrule
    $0.2$ & --- & ---
          & $0.08{\scriptstyle\pm0.11}$ & $1.89{\scriptstyle\pm0.13}$
          & $0.01{\scriptstyle\pm0.00}$ & $1.98{\scriptstyle\pm0.01}$
          & $0.34{\scriptstyle\pm0.44}$ & $\mathbf{1.41}{\scriptstyle\pm0.72}$
          & $0.19{\scriptstyle\pm0.16}$ & $1.65{\scriptstyle\pm0.16}$
          & $\mathbf{0.44}{\scriptstyle\pm0.24}$ & $\mathbf{1.25}{\scriptstyle\pm0.28}$ \\
    $0.5$ & $0.12{\scriptstyle\pm0.15}$ & $1.85{\scriptstyle\pm0.19}$
          & $0.15{\scriptstyle\pm0.20}$ & $1.75{\scriptstyle\pm0.30}$
          & $0.07{\scriptstyle\pm0.04}$ & $1.75{\scriptstyle\pm0.13}$
          & $\mathbf{0.70}{\scriptstyle\pm0.30}$ & $\mathbf{0.80}{\scriptstyle\pm0.46}$
          & $\mathbf{0.86}{\scriptstyle\pm0.13}$ & $\mathbf{0.70}{\scriptstyle\pm0.13}$
          & $\mathbf{0.95}{\scriptstyle\pm0.01}$ & $\mathbf{0.54}{\scriptstyle\pm0.08}$ \\
    $1$   & $0.12{\scriptstyle\pm0.15}$ & $1.83{\scriptstyle\pm0.19}$
          & $\mathbf{0.61}{\scriptstyle\pm0.26}$ & $\mathbf{1.08}{\scriptstyle\pm0.38}$
          & $\mathbf{0.61}{\scriptstyle\pm0.26}$ & $\mathbf{1.01}{\scriptstyle\pm0.27}$
          & $\mathbf{0.97}{\scriptstyle\pm0.03}$ & $\mathbf{0.36}{\scriptstyle\pm0.15}$
          & $\mathbf{0.97}{\scriptstyle\pm0.03}$ & $\mathbf{0.46}{\scriptstyle\pm0.07}$
          & $\mathbf{0.99}{\scriptstyle\pm0.00}$ & $\mathbf{0.40}{\scriptstyle\pm0.03}$ \\
    $2$   & $0.16{\scriptstyle\pm0.13}$ & $1.68{\scriptstyle\pm0.15}$
          & $\mathbf{0.92}{\scriptstyle\pm0.05}$ & $\mathbf{0.65}{\scriptstyle\pm0.22}$
          & $\mathbf{0.93}{\scriptstyle\pm0.05}$ & $\mathbf{0.45}{\scriptstyle\pm0.14}$
          & $\mathbf{0.99}{\scriptstyle\pm0.00}$ & $\mathbf{0.28}{\scriptstyle\pm0.06}$
          & $\mathbf{0.99}{\scriptstyle\pm0.01}$ & $\mathbf{0.37}{\scriptstyle\pm0.03}$
          & $\mathbf{0.99}{\scriptstyle\pm0.00}$ & $\mathbf{0.36}{\scriptstyle\pm0.01}$ \\
    $5$   & $\mathbf{0.93}{\scriptstyle\pm0.04}$ & $\mathbf{0.48}{\scriptstyle\pm0.14}$
          & $\mathbf{0.98}{\scriptstyle\pm0.00}$ & $\mathbf{0.41}{\scriptstyle\pm0.10}$
          & $\mathbf{0.99}{\scriptstyle\pm0.00}$ & $\mathbf{0.17}{\scriptstyle\pm0.05}$
          & $\mathbf{1.00}{\scriptstyle\pm0.00}$ & $\mathbf{0.28}{\scriptstyle\pm0.02}$
          & $\mathbf{0.99}{\scriptstyle\pm0.00}$ & $\mathbf{0.33}{\scriptstyle\pm0.01}$
          & $\mathbf{0.99}{\scriptstyle\pm0.00}$ & $\mathbf{0.34}{\scriptstyle\pm0.01}$ \\
    $10$  & $\mathbf{0.98}{\scriptstyle\pm0.01}$ & $\mathbf{0.28}{\scriptstyle\pm0.05}$
          & $\mathbf{0.99}{\scriptstyle\pm0.00}$ & $\mathbf{0.17}{\scriptstyle\pm0.04}$
          & $\mathbf{1.00}{\scriptstyle\pm0.00}$ & $\mathbf{0.09}{\scriptstyle\pm0.02}$
          & $\mathbf{1.00}{\scriptstyle\pm0.00}$ & $\mathbf{0.29}{\scriptstyle\pm0.01}$
          & $\mathbf{0.99}{\scriptstyle\pm0.00}$ & $\mathbf{0.32}{\scriptstyle\pm0.01}$
          & $\mathbf{0.99}{\scriptstyle\pm0.00}$ & $\mathbf{0.33}{\scriptstyle\pm0.01}$ \\
    \bottomrule
  \end{tabular}%
  }
\end{table}

\begin{table}[t]
  \caption{\textbf{OpenWebMath weight recovery under example-level DP with $K{=}16$ Common Pile proxies.} Mean $\pm$ std over 3 paired comparisons (3 seeds). We compare the private mixture weights $\boldsymbol{\lambda}_{\mathrm{DP}}$ to the non-private MixMin solution $\boldsymbol{\lambda}^{\star}$ using cosine similarity $\mathrm{cos}_\star$ and $L_1$ distance. Columns vary the DP-MixMin selection budget $\rho$, the fraction of the private mixture-learning split used to estimate the mixture weights; the second header row reports the corresponding number of OpenWebMath tokens. The uniform reference is $\mathrm{cos}_{\mathbf{u}}{=}0.40$ and $L_{1,\mathbf{u}}{=}1.28$. Bold values improve over the uniform reference on that metric.}
  \label{tab:owm-weights}
  \centering
  \scriptsize
  \setlength{\tabcolsep}{3pt}
  \resizebox{\linewidth}{!}{%
  \begin{tabular}{@{}l cc cc cc cc@{}}
    \toprule
    & \multicolumn{2}{c}{$\rho{=}0.01$} & \multicolumn{2}{c}{$\rho{=}0.02$} & \multicolumn{2}{c}{$\rho{=}0.05$} & \multicolumn{2}{c}{$\rho{=}0.1$} \\
    & \multicolumn{2}{c}{\scriptsize (275K tokens)} & \multicolumn{2}{c}{\scriptsize (552K tokens)} & \multicolumn{2}{c}{\scriptsize (1.4M tokens)} & \multicolumn{2}{c}{\scriptsize (2.8M tokens)} \\
    \cmidrule(lr){2-3} \cmidrule(lr){4-5} \cmidrule(lr){6-7} \cmidrule(lr){8-9}
    $\varepsilon$ & $\mathrm{cos}_\star$ & $L_1$ & $\mathrm{cos}_\star$ & $L_1$ & $\mathrm{cos}_\star$ & $L_1$ & $\mathrm{cos}_\star$ & $L_1$ \\
    \midrule
    $0.2$ & $0.17{\scriptstyle\pm0.08}$ & $1.66{\scriptstyle\pm0.14}$ & $0.34{\scriptstyle\pm0.17}$ & $1.38{\scriptstyle\pm0.23}$ & $\mathbf{0.78}{\scriptstyle\pm0.10}$ & $\mathbf{0.88}{\scriptstyle\pm0.14}$ & $\mathbf{0.92}{\scriptstyle\pm0.02}$ & $\mathbf{0.72}{\scriptstyle\pm0.04}$ \\
    $0.4$ & $0.32{\scriptstyle\pm0.14}$ & $1.39{\scriptstyle\pm0.20}$ & $\mathbf{0.68}{\scriptstyle\pm0.13}$ & $\mathbf{0.99}{\scriptstyle\pm0.16}$ & $\mathbf{0.92}{\scriptstyle\pm0.02}$ & $\mathbf{0.72}{\scriptstyle\pm0.04}$ & $\mathbf{0.94}{\scriptstyle\pm0.00}$ & $\mathbf{0.67}{\scriptstyle\pm0.01}$ \\
    $0.6$ & $\mathbf{0.49}{\scriptstyle\pm0.15}$ & $\mathbf{1.19}{\scriptstyle\pm0.20}$ & $\mathbf{0.82}{\scriptstyle\pm0.07}$ & $\mathbf{0.84}{\scriptstyle\pm0.10}$ & $\mathbf{0.94}{\scriptstyle\pm0.01}$ & $\mathbf{0.69}{\scriptstyle\pm0.02}$ & $\mathbf{0.95}{\scriptstyle\pm0.00}$ & $\mathbf{0.67}{\scriptstyle\pm0.01}$ \\
    $0.8$ & $\mathbf{0.64}{\scriptstyle\pm0.12}$ & $\mathbf{1.02}{\scriptstyle\pm0.15}$ & $\mathbf{0.88}{\scriptstyle\pm0.04}$ & $\mathbf{0.77}{\scriptstyle\pm0.06}$ & $\mathbf{0.94}{\scriptstyle\pm0.01}$ & $\mathbf{0.68}{\scriptstyle\pm0.01}$ & $\mathbf{0.95}{\scriptstyle\pm0.00}$ & $\mathbf{0.67}{\scriptstyle\pm0.01}$ \\
    $1$   & $\mathbf{0.73}{\scriptstyle\pm0.09}$ & $\mathbf{0.92}{\scriptstyle\pm0.12}$ & $\mathbf{0.91}{\scriptstyle\pm0.02}$ & $\mathbf{0.74}{\scriptstyle\pm0.04}$ & $\mathbf{0.94}{\scriptstyle\pm0.00}$ & $\mathbf{0.67}{\scriptstyle\pm0.00}$ & $\mathbf{1.00}{\scriptstyle\pm0.00}$ & $\mathbf{0.05}{\scriptstyle\pm0.03}$ \\
    $2$   & $\mathbf{0.90}{\scriptstyle\pm0.01}$ & $\mathbf{0.74}{\scriptstyle\pm0.04}$ & $\mathbf{0.94}{\scriptstyle\pm0.00}$ & $\mathbf{0.68}{\scriptstyle\pm0.02}$ & $\mathbf{0.95}{\scriptstyle\pm0.00}$ & $\mathbf{0.67}{\scriptstyle\pm0.01}$ & $\mathbf{1.00}{\scriptstyle\pm0.00}$ & $\mathbf{0.04}{\scriptstyle\pm0.01}$ \\
    $5$   & $\mathbf{0.94}{\scriptstyle\pm0.01}$ & $\mathbf{0.67}{\scriptstyle\pm0.01}$ & $\mathbf{0.94}{\scriptstyle\pm0.01}$ & $\mathbf{0.67}{\scriptstyle\pm0.02}$ & $\mathbf{0.95}{\scriptstyle\pm0.00}$ & $\mathbf{0.67}{\scriptstyle\pm0.01}$ & $\mathbf{1.00}{\scriptstyle\pm0.00}$ & $\mathbf{0.03}{\scriptstyle\pm0.01}$ \\
    \bottomrule
  \end{tabular}%
  }
\end{table}

\begin{table}[t]
\caption{\textbf{Empirical privacy audit of DP-MixMin's mixture-learning step, CheXpert $\to$ NIH ChestX-ray14, at $\boldsymbol{\varepsilon=1}$.} We audit the released mixture weights $\boldsymbol{\lambda}=M_{\mathrm{mix}}(D_{\mathrm{mix}})$, reporting the Clopper--Pearson empirical lower bound $\varepsilon_{\mathrm{LB}}$ at $\varepsilon=1$, $\delta=10^{-5}$, $\rho=0.1$, and $k=14$ proxy sources, over \textbf{2000} runs on add/remove-neighboring datasets (attack fit on a SELECT half, bound computed on a disjoint EVAL half; $95\%$ confidence). The \emph{natural} canary is a real held-out example; the \emph{max-clip} canary injects a worst-case gradient at the clipping norm. \emph{Private} audits at $\varepsilon=1$; \emph{No-noise control} removes the DP noise. Every private $\varepsilon_{\mathrm{LB}}\ll 1$ while controls recover large leakage, confirming audit fidelity.}
  \label{tab:audit-chexpert}
  \centering
  \scriptsize
  \setlength{\tabcolsep}{5pt}
  \begin{tabular}{@{}l c cc cc@{}}
    \toprule
    & & \multicolumn{2}{c}{\textbf{Natural canary}} & \multicolumn{2}{c}{\textbf{Max-clipped-gradient canary}} \\
    \cmidrule(lr){3-4} \cmidrule(lr){5-6}
    \textbf{Class} & \textbf{\% pos} & Private ($\varepsilon{=}1$) & No-noise ctrl. & Private ($\varepsilon{=}1$) & No-noise ctrl. \\
    \midrule
    Atelectasis        & 10.31 & 0.000 & 0.261 & 0.080 & 3.71 \\
    Consolidation      &  4.16 & 0.000 & 0.501 & 0.052 & 5.19 \\
    Infiltration       & 17.74 & 0.000 & 0.088 & 0.000 & 2.27 \\
    Pneumothorax       &  4.73 & 0.000 & 0.068 & 0.080 & 5.18 \\
    Edema              &  2.05 & 0.000 & 0.404 & 0.000 & 4.93 \\
    Emphysema          &  2.24 & 0.000 & 0.619 & 0.054 & 4.93 \\
    Fibrosis           &  1.50 & 0.000 & 1.188 & 0.091 & 5.18 \\
    Effusion           & 11.88 & 0.000 & 0.000 & 0.028 & 3.73 \\
    Pneumonia          &  1.28 & 0.000 & 0.344 & 0.098 & 5.60 \\
    Pleural Thickening &  3.02 & 0.000 & 1.226 & \textbf{0.131} & 5.60 \\
    Cardiomegaly       &  2.48 & 0.000 & 0.148 & 0.000 & 4.93 \\
    Nodule             &  5.65 & 0.000 & 0.539 & 0.000 & 5.60 \\
    Mass               &  5.16 & 0.000 & 0.000 & 0.000 & 5.60 \\
    Hernia             &  0.20 & 0.000 & 2.360 & 0.086 & 5.60 \\
    \midrule
    \textbf{max over tasks} & --- & 0.000 & 2.360 & \textbf{0.131} & 5.60 \\
    \bottomrule
  \end{tabular}
\end{table}

\begin{table}[t]
\caption{\textbf{Empirical privacy audit of DP-MixMin's mixture-learning step, MIMIC-CXR $\to$ NIH ChestX-ray14, at $\boldsymbol{\varepsilon=1}$.} Identical protocol to Table~\ref{tab:audit-chexpert}: Clopper--Pearson $\varepsilon_{\mathrm{LB}}$ at $\varepsilon=1$, $\delta=10^{-5}$, $\rho=0.1$, and $k=14$ MIMIC-CXR proxy sources, over 2000 neighboring-dataset runs (SELECT/EVAL split, $95\%$ confidence). As for CheXpert, every private $\varepsilon_{\mathrm{LB}}\ll 1$ while controls recover substantial leakage.}
  \label{tab:audit-mimic}
  \centering
  \scriptsize
  \setlength{\tabcolsep}{5pt}
  \begin{tabular}{@{}l c cc cc@{}}
    \toprule
    & & \multicolumn{2}{c}{\textbf{Natural canary}} & \multicolumn{2}{c}{\textbf{Max-clipped-gradient canary}} \\
    \cmidrule(lr){3-4} \cmidrule(lr){5-6}
    \textbf{Class} & \textbf{\% pos} & Private ($\varepsilon{=}1$) & No-noise ctrl. & Private ($\varepsilon{=}1$) & No-noise ctrl. \\
    \midrule
    Atelectasis        & 10.31 & 0.000 & 0.000 & 0.000 & 4.02 \\
    Consolidation      &  4.16 & 0.000 & 0.501 & 0.000 & 4.72 \\
    Infiltration       & 17.74 & 0.000 & 0.000 & 0.001 & 3.21 \\
    Pneumothorax       &  4.73 & 0.000 & 0.022 & 0.000 & 4.92 \\
    Edema              &  2.05 & 0.000 & 0.672 & 0.000 & 5.59 \\
    Emphysema          &  2.24 & 0.000 & 0.000 & 0.000 & 5.40 \\
    Fibrosis           &  1.50 & 0.000 & 2.627 & 0.056 & 4.44 \\
    Effusion           & 11.88 & 0.000 & 0.000 & 0.000 & 2.87 \\
    Pneumonia          &  1.28 & 0.000 & 0.480 & 0.000 & 5.17 \\
    Pleural Thickening &  3.02 & 0.000 & 1.147 & 0.000 & 5.07 \\
    Cardiomegaly       &  2.48 & 0.000 & 1.895 & \textbf{0.101} & 5.51 \\
    Nodule             &  5.65 & 0.000 & 0.582 & 0.000 & 4.73 \\
    Mass               &  5.16 & 0.000 & 0.551 & 0.000 & 4.76 \\
    Hernia             &  0.20 & 0.000 & 4.934 & 0.061 & 4.92 \\
    \midrule
    \textbf{max over tasks} & --- & 0.000 & 4.934 & \textbf{0.101} & 5.51 \\
    \bottomrule
  \end{tabular}
\end{table}

\begin{table}[t]
  \caption{\textbf{DP-MixMin vs.\ its non-private version, CheXpert $\to$ NIH ChestX-ray14 at $\boldsymbol{\varepsilon=1}$.} Per-class final AUC (mean $\pm$ std over 3--5 seeds, $\delta=10^{-5}$). Classes are ordered by NIH index (Table~\ref{tab:nih-chexpert-mapping}). DP-MixMin matches the non-private upper bound to within $0.001$ macro AUC.}
  \label{tab:chexpert-nonpriv-vs-dp}
  \centering
  \small
  \setlength{\tabcolsep}{6pt}
  \begin{tabular}{@{}l cc@{}}
    \toprule
    \textbf{Class} & \textbf{Non-private MixMin} & \textbf{DP-MixMin ($\varepsilon{=}1$)} \\
    \midrule
    Atelectasis        & $0.745 \pm 0.007$ & $0.747 \pm 0.006$ \\
    Consolidation      & $0.753 \pm 0.014$ & $0.756 \pm 0.012$ \\
    Infiltration       & $0.669 \pm 0.005$ & $0.669 \pm 0.006$ \\
    Pneumothorax       & $0.796 \pm 0.008$ & $0.789 \pm 0.018$ \\
    Edema              & $0.819 \pm 0.012$ & $0.809 \pm 0.019$ \\
    Emphysema          & $0.715 \pm 0.026$ & $0.746 \pm 0.025$ \\
    Fibrosis           & $0.645 \pm 0.024$ & $0.655 \pm 0.019$ \\
    Effusion           & $0.850 \pm 0.006$ & $0.851 \pm 0.005$ \\
    Pneumonia          & $0.595 \pm 0.015$ & $0.592 \pm 0.019$ \\
    Pleural Thickening & $0.688 \pm 0.017$ & $0.671 \pm 0.022$ \\
    Cardiomegaly       & $0.878 \pm 0.007$ & $0.868 \pm 0.013$ \\
    Nodule             & $0.585 \pm 0.012$ & $0.570 \pm 0.012$ \\
    Mass               & $0.691 \pm 0.004$ & $0.670 \pm 0.011$ \\
    Hernia             & $0.545 \pm 0.030$ & $0.567 \pm 0.029$ \\
    \midrule
    \textbf{Macro AUC} & \textbf{0.712} & \textbf{0.711} \\
    \bottomrule
  \end{tabular}
\end{table}

\begin{table}[t]
  \caption{\textbf{DP-MixMin vs.\ its non-private version, MIMIC-CXR $\to$ NIH ChestX-ray14 at $\boldsymbol{\varepsilon=1}$.} Per-class final AUC (mean $\pm$ std over seeds, $\delta=10^{-5}$). \emph{Non-private MixMin} finds the mixture with no privacy and fine-tunes on the \emph{full} private set (both splits), so it is a strict upper bound; DP-MixMin reserves the mixing split (parallel composition) and privatizes both steps. DP-MixMin matches the non-private upper bound to within $0.006$ macro AUC.}
  \label{tab:mimic-nonpriv-vs-dp}
  \centering
  \small
  \setlength{\tabcolsep}{6pt}
  \begin{tabular}{@{}l cc@{}}
    \toprule
    \textbf{Class} & \textbf{Non-private MixMin} & \textbf{DP-MixMin ($\varepsilon{=}1$)} \\
    \midrule
    Atelectasis        & $0.758 \pm 0.006$ & $0.749 \pm 0.010$ \\
    Consolidation      & $0.764 \pm 0.015$ & $0.761 \pm 0.013$ \\
    Infiltration       & $0.678 \pm 0.008$ & $0.676 \pm 0.003$ \\
    Pneumothorax       & $0.779 \pm 0.004$ & $0.754 \pm 0.014$ \\
    Edema              & $0.838 \pm 0.007$ & $0.835 \pm 0.010$ \\
    Emphysema          & $0.745 \pm 0.018$ & $0.720 \pm 0.028$ \\
    Fibrosis           & $0.654 \pm 0.024$ & $0.672 \pm 0.029$ \\
    Effusion           & $0.855 \pm 0.005$ & $0.851 \pm 0.004$ \\
    Pneumonia          & $0.656 \pm 0.022$ & $0.658 \pm 0.016$ \\
    Pleural Thickening & $0.687 \pm 0.023$ & $0.692 \pm 0.014$ \\
    Cardiomegaly       & $0.850 \pm 0.011$ & $0.832 \pm 0.030$ \\
    Nodule             & $0.613 \pm 0.007$ & $0.580 \pm 0.007$ \\
    Mass               & $0.700 \pm 0.017$ & $0.675 \pm 0.013$ \\
    Hernia             & $0.543 \pm 0.064$ & $0.584 \pm 0.069$ \\
    \midrule
    \textbf{Macro AUC} & \textbf{0.723} & \textbf{0.717} \\
    \bottomrule
  \end{tabular}
\end{table}

\section{Dataset Access and Asset Licenses}
\label{app:asset_licenses}

All restricted or credentialed datasets were accessed under their applicable
data-use terms. For PhysioNet credentialed-access datasets, the authors completed
the required credentialing process and signed the applicable Data Use Agreement
before accessing the data. We do not redistribute restricted datasets.

\paragraph{NIH Chest X-rays.}
Accessed under the CCO: public domain license which gives free-use. The dataset is available at \url{https://www.kaggle.com/datasets/nih-chest-xrays/data}

\paragraph{CheXpert.}
Accessed under the Stanford CheXpert Dataset Research Use Agreement, which
permits non-commercial research use. The dataset is available at
\url{https://stanfordmlgroup.github.io/competitions/chexpert}.

\paragraph{MIMIC-CXR.}
Accessed under the PhysioNet Credentialed Health Data License~1.5.0.
Access requires completion of a credentialing process and signing a Data Use
Agreement through PhysioNet~\citep{johnson2019mimiccxr}.

\paragraph{Enron Email Dataset.}
The Enron Corpus was released into the public domain following federal litigation and is freely available for research use~\citep{klimt2004introducing}.

\paragraph{The Common Pile.}
Released under the Open Data Commons Attribution License (ODC-By-1.0). Individual constituent sources carry their own licenses; we refer readers to the dataset documentation for a full provenance breakdown~\citep{kandpal2025commonpilev018tb}.